\documentclass{article}

\usepackage[preprint]{neurips_2021}
\usepackage[bottom]{footmisc}
\usepackage[pdftex]{graphicx}
\usepackage{multicol}
\usepackage{wrapfig}
\usepackage[utf8]{inputenc} % allow utf-8 input
\usepackage[T1]{fontenc}    % use 8-bit T1 fonts
\usepackage{url}            % simple URL typesetting
\usepackage{booktabs}       % professional-quality tables
\usepackage{amsfonts}       % blackboard math symbols
\usepackage{nicefrac}       % compact symbols for 1/2, etc.
\usepackage{graphicx,bm,times}
\usepackage{microtype}      % microtypography
\usepackage{amsmath}
\usepackage{subfig}
\usepackage{multicol}
\usepackage{amssymb}
\usepackage{booktabs}       % professional-quality tables
\usepackage{amsfonts}       % blackboard math symbols
\usepackage{nicefrac}       % compact symbols for 1/2, etc.
\usepackage{microtype}      % microtypography

\usepackage[english]{babel}
\usepackage{amsthm}
\usepackage[dvipsnames]{xcolor}
\newtheorem{theorem}{Theorem}
\newtheorem{proposition}{Proposition}
\newtheorem{definition}{Definition}
\theoremstyle{plain}
\usepackage{tikz}

\def\RR{{\mathbb R}}    %r�els
\def\PP{{\mathbb P}}     %espace projectif / probabilité
\def\EE{{\mathbb E}}    % espérance
\def\11{{\mathbf 1}}    % indicatrice

  \def\cG{{\mathcal G}}     \def\cH{{\mathcal H}}  \def\cT{{\mathcal T}}      \def\cD{{\mathcal D}}   \def\cP{{\mathcal P}}           \def\cX{{\mathcal X}} \def\cY{{\mathcal Y}}  \def\cZ{{\mathcal Z}}

\def\bfx{{\bf x}} \def\bfy{{\bf y}} \def\bfz{{\bf z}}

\def\bft{{\bf t}}

\def\tilmu{{\tilde{\mu}}}
\def\tilf{\tilde{f}}

\def\tily{{\tilde{y}}}
\def\tilbfy{{\bf \tilde{y}}}

\newcommand\blankfootnote[1]{%
  \let\thefootnote\relax\footnotetext{#1}%
  \let\thefootnote\svthefootnote%
}
\let\svfootnote\footnote
\renewcommand\footnote[2][?]{%
  \if\relax#1\relax%
    \blankfootnote{#2}%
  \else%
    \if?#1\svfootnote{#2}\else\svfootnote[#1]{#2}\fi%
  \fi
}

\theoremstyle{plain}
    \newtheorem{innercustomgeneric}{\customgenericname}
    \providecommand{\customgenericname}{}
    \newcommand{\newcustomtheorem}[2]{%
      \newenvironment{#1}[1]
      {%
       \renewcommand\customgenericname{#2}%
       \renewcommand\theinnercustomgeneric{##1}%
       \innercustomgeneric
      }
      {\endinnercustomgeneric}
    }
    
    \newcustomtheorem{customthm}{Theorem}
    \newcustomtheorem{customlemma}{Lemma}
    \newcustomtheorem{customprop}{Proposition}

\newtheoremstyle{TheoremNum}
    {\topsep}{\topsep}              %%% space between body and thm
    {\itshape}                      %%% Thm body font
    {}                              %%% Indent amount (empty = no indent)
    {\bfseries}                     %%% Thm head font
    {.}                             %%% Punctuation after thm head
    { }                             %%% Space after thm head
    {\thmname{#1}\thmnote{ \bfseries #3}}%%% Thm head spec
\theoremstyle{TheoremNum}

\usepackage[english]{babel}

\addto\captionsenglish{% Replace "english" with the language you use
  \renewcommand{\contentsname}%
    {Supplementary materials for \textsc{BayesIMP: }Uncertainty Quantification for Causal Data Fusion}%
}

\title{\Large{\textsc{BayesIMP: }Uncertainty Quantification for Causal Data Fusion}}

\author{
    Siu Lun Chau$^*$ \\
    University of Oxford\\
    \And
    Jean-Fran\c cois Ton$^*$ \\
    University of Oxford\\
    \And
    Javier Gonz\'alez \\
    Microsoft Research Cambridge\\
    \And
    Yee Whye Teh \\
    University of Oxford\\
    \And
    Dino Sejdinovic \\
    University of Oxford\\
}
\begin{document}
\maketitle
\begin{abstract}
While causal models are becoming one of the mainstays of machine learning, the problem of uncertainty quantification in causal inference remains challenging. In this paper, we study the causal data fusion problem, where datasets pertaining to multiple causal graphs are combined to estimate the average treatment effect of a target variable. As data arises from multiple sources and can vary in quality and quantity, principled uncertainty quantification becomes essential. To that end, we introduce Bayesian Interventional Mean Processes, a framework which combines ideas from probabilistic integration and kernel mean embeddings to represent interventional distributions in the reproducing kernel Hilbert space, while taking into account the uncertainty within each causal graph. To demonstrate the utility of our uncertainty estimation, we apply our method to the Causal Bayesian Optimisation task and show improvements over state-of-the-art methods. %
\footnote[]{$^*$\text{Denotes equal contribution with alphabetical ordering}}
\end{abstract}

% \addtocontents{toc}{\protect\setcounter{tocdepth}{-1}}
\section{Introduction}\label{intro}

%\subsection*{The task}
Causal inference has seen a significant surge of research interest
in areas such as healthcare \cite{thompson2019causal}, ecology \cite{courtney2017environmental}, and optimisation \cite{aglietti2020multi}. However, data fusion, the problem of merging information from multiple data sources, has received limited attention in the context of causal modelling, yet presents significant potential benefits for practical situations  \cite{meng2020survey, singh2019kernel}. In this work, we consider a causal data fusion problem where two causal graphs are combined for the purposes of inference of a target variable (see Fig.\ref{fig:medical example}). In particular, our goal is to quantify the uncertainty under such a setup and determine the level of confidence in our treatment effect estimates.

%Different from the causal data fusion problem discussed in \cite{bareinboim2016causal}, who considered extrapolating experimental findings across treatment domains, i.e inferring $\EE[Y|do(X)]$ when only experimental data from $p(Y|do(Z))$ is observed, we focus on combining causal graphs (see Fig.\ref{fig:medical example}) instead. In particular, our goal is to quantify the uncertainty under such setup, to determine the level of confidence in our treatment effects estimates. \yw{this paragraph confusing. just describe what this paper is about, rather than starting with what this paper is not about [6], or talking about data fusion in general. come back to these later in intro, or in related works. notations undefined.}
\begin{figure}[!htp]
    \centering
    \includegraphics[width=\textwidth]{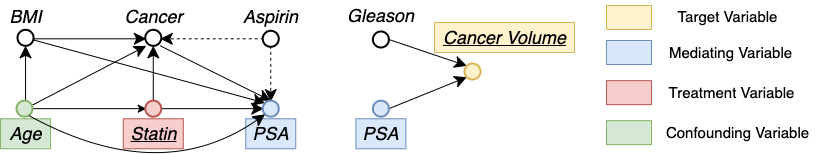}
    \caption{Example problem setup: Causal graphs collected in two separate medical studies i.e. \cite{ferro2015use} and \cite{stamey1989prostate}. (Left) $\cD_1:$ Data describing the causal relationships between statin level and Prostate Specific Antigen (PSA). (Right) $\cD_2:$ Data from a prostate cancer study for patients about to receive a radical prostatectomy. Goal: \textbf{Model} \textbf{$\pmb{\mathbb{E}}[\emph{Cancer Volume}|\emph{do(Statin)}]$} while also quantifying its uncertainty.}
    \label{fig:medical example}
\end{figure}

Let us consider the motivating example in Fig.\ref{fig:medical example}, where a medical practitioner is investigating how \emph{prostate cancer volume} is affected by a \textit{statin} drug dosage. We consider the case where the doctor only has access to two separate medical studies describing the quantities of interest. On one hand we have observational data, from one medical study $\cD_1$ \cite{thompson2019causal}, describing the causal relationship between \textit{statin} level and \textit{prostate specific antigen (PSA)}, and on the other hand we have observational data, from a second study $\cD_2$ \cite{stamey1989prostate}, that looked into the link between \textit{PSA} level and \textit{prostate cancer volume}. The goal is to model the \textbf{interventional effect} between our target variable (\textit{cancer volume}) and the treatment variable (\textit{statin}). This problem setting is different from the standard observational scenario as it comes with the following challenges:

\begin{itemize}
    \item \textbf{Unmatched data:} Our goal is to estimate $\EE[\textit{cancer volume}|do(\textit{statin})]$ but the observed \textit{cancer volume} is not paired with \textit{statin} dosage. Instead, they are related via a mediating variable \textit{PSA}.
    \item \textbf{Uncertainty quantification:} The two studies may be of different data quantity/quality. Furthermore, a covariate shift in the mediating variable, i.e. a difference between its distributions in two datasets, may cause inaccurate extrapolation. Hence, we need to account for uncertainty in both datasets.
\end{itemize}

Formally, let $X$ be the treatment (\emph{Statin}), $Y$ be the mediating variable (\emph{PSA}) and $T$ our target (\emph{cancer volume}), and our aim is to estimate $\EE[T|do(X)]$. 
The problem of unmatched data in a similar context has been previously considered by \cite{singh2019kernel} using a two-staged regression approach ($X\to Y$ and $Y\to T$). However, uncertainty quantification, despite being essential if our estimates of interventional effects will guide decision-making, has not been previously explored. In particular, it is crucial to quantify the  uncertainty in both stages as this takes into account the lack of data in specific parts of the space. Given that we are using different datasets for each stage, there are also two sources of epistemic uncertainties (due to lack of data) as well as two sources of aleatoric uncertainties (due to inherent randomness in $Y$ and $T$) \cite{hullermeier2021aleatoric} . It is thus natural to consider regression models based on Gaussian Processes (GP) \cite{rasmussen2003gaussian}, as they are able to model both types of uncertainties. However, as GPs, or any other standard regression models, are designed to model conditional expectations only and will fail to capture the underlying distributions of interest (e.g.\ if there is multimodality in $Y$ as discussed in \cite{pmlr-v130-ton21a}). This is undesirable since, as we will see, interventional effect estimation requires accurate estimates of distributions. While one could in principle resort to density estimation methods, this becomes challenging since we typically deal with a number of conditional/ interventional densities. %While explicit density estimation methods can be used to model the densities, they may suffer from model misspecification and fail to accurately capture the underlying distribution of interest (e.g. if there is multimodality as discussed in \cite{pmlr-v130-ton21a}). Furthermore, these methods are known in the inefficient as discussed in [add KCI ref].

% in which arsenal of kernel methods can be used for probabilistic modelling.
In this paper, we introduce the framework of \emph{Bayesian Interventional Mean Processes} (\textsc{BayesIMP}) to circumvent the  challenges in the causal data fusion setting described above. 
\textsc{BayesIMP} considers kernel mean embeddings \cite{muandet2017kernel} for representing distributions in a reproducing kernel Hilbert space (RKHS), in which the whole arsenal of kernel methods can be extended to probabilistic inference (e.g.\ kernel Bayes rule \cite{fukumizu2010kernel}, hypothesis testing \cite{zhang2018large}, distribution regression \cite{law2018bayesian}). 
Specifically, \textsc{BayesIMP} uses
kernel mean embeddings 
to represent the interventional distributions and
to analytically marginalise out $Y$, hence accounting for aleatoric uncertainties.
Further, \textsc{BayesIMP}
uses GPs to estimate the required kernel mean embeddings from data in a Bayesian manner,
which allows to quantify the epistemic uncertainties when representing the interventional distributions.
To illustrate the quality of our uncertainty estimates, we apply \textsc{BayesIMP} to Causal Bayesian Optimisation \cite{aglietti2020causal}, an efficient heuristic to optimise objective functions of the form  $x^* = \arg \min_{x\in\mathcal{X}} \EE[T|do(X)=x]$.  Our contributions are summarised below:
\begin{enumerate}
    \item We propose a novel \textit{Bayesian Learning of Conditional Mean Embedding} (\textsc{BayesCME}) that allows us to estimate conditional mean embeddings in a Bayesian framework.
    \item Using \textsc{BayesCME}, we propose a novel \emph{Bayesian Interventional Mean Process} (\textsc{BayesIMP}) that allows us to model interventional effect across causal graphs without explicit density estimation, while obtaining uncertainty estimates for $\EE[T|do(X)=x]$.
    \item We apply \textsc{BayesIMP} to Causal Bayesian Optimisation, a problem introduced in \cite{aglietti2020causal} and show significant improvements over existing state-of-the-art methods.
\end{enumerate}

Note that \cite{bareinboim2016causal} also considered a causal fusion problem but with a different objective. They focused on extrapolating experimental findings across treatment domains, i.e.\ inferring $\EE[Y|do(X)]$ when only data from $p(Y|do(S))$ is observed, where $S$ is some other treatment variable. In contrast, we focus on modelling combined causal graphs, with a strong emphasis on uncertainty quantification.
While \cite{singh2020kernel} considered mapping interventional distributions in the RKHS to model quantities such as $\EE[T|do(X)]$, they only considered a frequentist approach, which does not account for epistemic uncertainties.

%Given that we use a GP-based model, we are able to capture both epistemic and aleatoric uncertainty. 

% \textbf{Firstly}, we propose a novel \textit{Bayesian Learning of Conditional Mean Embedding} that allows us to estimate conditional mean embeddings in a Bayesian framework.

% \textbf{Secondly}, using our \textit{Bayesian Learning of Conditional Mean Embedding}, we propose a novel \textit{two-staged Causal Gaussian Process} (BayesCMP) that allows us to model interventional effect across multiple causal graphs without explicit density estimators, while obtaining uncertainty estimates.

% \textbf{Lastly}, we apply BayesCMP to Causal Bayesian Optimisation and show significant improvements over existing state-of-the-art methods.

% \paragraph{Outline} The paper is outlined as follows: In section 2, we review some background material on causal inference \cite{pearl1995causal} and conditional mean embeddings \cite{song2013kernel}. In particular, we take a closer look at the interventional mean embedding proposed in \cite{singh2019kernel}, which is an important component for our algorithm. In section 3, we introduce our main result \textsc{BayesCMP}, along with a novel Bayesian version of conditional mean embedding. In section 4, we show experiments on synthetic and real-world data, where we significantly improve upon existing methods. Finally we conclude the work with a discussion and future direction in section 5.
\newpage
\paragraph{Notations.}
\begin{wrapfigure}{r}{0.35\textwidth}
    \centering
    \includegraphics[width=0.35\textwidth]{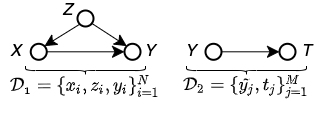}
    \caption{A general two stage causal learning setup. }
    \label{fig: setup_notation}
    \vspace{-0.7cm}
\end{wrapfigure}
We denote $X, Y, Z$ as random variables taking values in the non-empty sets $\cX, \cY$ and $\cZ$ respectively. Let $k_x: \cX \times \cX \rightarrow \RR$ be positive definite kernels on $X$ with an associated RKHS $\cH_{k_x}$. The corresponding canonical feature map $k_x(x', \cdot)$ is denoted as $\phi_x(x')$. Analogously for $Y$ and $Z$.

In the simplest setting, we observe i.i.d samples $\cD_1 = \{x_i, y_i, z_i\}_{i=1}^N$ from joint distribution $\PP_{XYZ}$ which we concatenate into vectors $\bfx := [x_1, ..., x_N]^\top$. Similarly for $\bfy$ and $\bfz$. For this work, $X$ is referred as \textit{treatment variable}, $Y$ as \textit{mediating variable} and $Z$ as \emph{adjustment variables} accounting for confounding effects. With an abuse of notation, features matrices are defined by stacking feature maps along the columns, i.e $\Phi_{\bfx} := [\phi_x(x_1), ..., \phi_x(x_N)]$. We denote the Gram matrix as $K_{\bfx\bfx} := \Phi_{\bfx}^\top \Phi_{\bfx}$ and the vector of evaluations $k_{x\bfx}$ as $[k_x(x, x_1), ..., k_x(x, x_N)]$. We define $\Phi_{\bfy}, \Phi_{\bfz}$ analogously for $\bfy$ and $\bfz$. 

% $\bfy := [y_1, ..., y_N]^\top$ and $\bfz := [z_1, ..., z_N]^\top$

Lastly, we denote $T = f(Y) + \epsilon$ as our \emph{target variable}, which is modelled as some noisy evaluation of a function $f: \cY \rightarrow \cT$ on $Y$ while $\epsilon$ being some random noise. For our problem setup we observe a second dataset of i.i.d realisations $\cD_2 = \{\tilde{y}_j, t_j\}_{j=1}^M$ from the joint $\PP_{YT}$ independent of $\cD_1$. Again, we define $\tilbfy := [\tily_1, ..., \tily_M]^\top$ and $\bft := [t_1,..., t_M]^\top$ just like for $\cD_1$. See Fig.\ref{fig: setup_notation} for illustration. %Equivalently the feature  is defined as $\Phi_{\tilbfy} := [k_y(\tily_1, \cdot), ..., k_y(\tily_M, \cdot)]$. 

\section{Background} \label{sec: background}
% \vspace{-0.2cm}

Representing interventional distributions in an RKHS has been explored in different contexts  \cite{muandet2018counterfactual, singh2020kernel, mitrovic2018causal}. In particular, when the treatment is continuous, \cite{singh2020kernel} introduced the \emph{Interventional Mean Embeddings} (\textsc{IME}s) to model densities in an RKHS by utilising smoothness across treatments. Given that \textsc{IME} is an important building block to our contribution, we give it a detailed review by first introducing the key concepts of \textit{do}-calculus \citep{pearl1995causal} and conditional mean embeddings \citep{song2013kernel}.

\subsection{Interventional distribution and \textit{do}-calculus}
In this work, we consider the structural causal model \cite{pearl1995causal} (SCM) framework, where a causal directed acyclic graph (DAG) $\cG$ is given and encodes knowledge of existing causal mechanisms amongst the variables in terms of conditional independencies. Given random variables $X$ and $Y$, a central question in interventional inference \cite{pearl1995causal} is to estimate the distribution $p(Y|do(X)=x)$, where $\{do(X)=x\}$ represents an intervention on $X$ whose value is set to $x$.  %(the density that accounts for confounding variables for example). 
Note that this quantity is not directly observed given that we are usually only given observational data, i.e, data sampled from the conditional $p(Y|X)$ but not from the interventional density $p(Y|do(X))$. However, Pearl \cite{pearl1995causal} developed \textit{do}-calculus which allows us to estimate interventional distributions from purely observational distributions under the identifiability assumption. Here we present the backdoor and front-door adjustments, which are the fundamental components of DAG based causal inference.

\label{sec:causal}
%Consider a structural causal model with of a directed acyclic causal graph $\cG$, and a triplet $\langle U, V, F \rangle$ \cite{pearl2009causality}, where $U$ is a set of independent \textit{exogenous} background/noise variables, $V$ a set of \textit{endogenous} variables and $F=\{f_1,..., f_{|V|}\}$ a set of functions determining how variables are generated, e.g $v_i = f(pa_i, u_i)$ where $pa_i$ denotes the parents of $V_i$ in $\cG$. The causal graph $\cG$ encodes our knowledge of existing causal mechanisms among $V$ in the form of conditional independence. We assume in our problem the causal graph is known, following the setup from \cite{aglietti2020causal}.

%Given sets $X, Y \subseteq V$, a central question in causal inference is that of estimating the interventional distribution $p(Y|do(X)=x)$. If the causal effect is identifiable, one can apply the 3 rules of \textit{do}-calculus \cite{pearl1995causal} to estimate these interventional distribution purely from observational data. For more details on identifiability and \textit{do}-calculus, we refer the reader to \citep{pearl1995causal}.

\begin{wrapfigure}{r}{0.25\textwidth}
    \centering
    \includegraphics[width=0.25\textwidth]{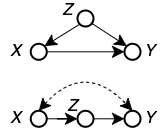}
    \caption{(Top) Backdoor adjustment (Bottom) Front-door adjustment, dashed edges denote unobserved confounders. }
    \label{adjusment_illus}
    \vspace*{-0.5cm}
\end{wrapfigure}
The backdoor adjustment is applicable when there are observed confounding variables $Z$ between the cause $X$ and the effect $Y$ (see Fig. \ref{adjusment_illus} (Top)). In order to correct for this confounding bias we can use the following equation, adjusting for $Z$ as $p(Y|do(X)=x) = \int_{\cZ} p(Y|X=x, z) p(z) dz$.
% \begin{equation}
%     p(Y|do(X)=x) = \int_{\cZ} p(Y|X=x, z) p(z) dz.
% \end{equation}

The front-door adjustment applies to cases when confounders are unobserved (see Fig. \ref{adjusment_illus} (Bottom)). Given a set of front-door adjustment variables $Z$, we can again correct the estimate for the causal effect from $X$ to $Y$ with $p(Y|do(X)=x) = \int_{\cZ}\int_{\cX} p(Y|x', z)p(z|X=x)p(x') dx' dz$.
% \begin{align}
%     p(Y|do(X)=x) = \int_{\cZ}\int_{\cX} p(Y|x', z)p(z|X=x)p(x') dx' dz.
% \end{align}

We rewrite the above formulae in a more general form as we show below. For the remainder of the paper we will opt for this notation:
\begin{align}
    p(Y|do(X)=x) = \EE_{\Omega_x}[p(Y|\Omega_x)] = \int
    p(Y|\Omega_x) p(\Omega_x) d\Omega_x
\end{align}
For backdoor we have $\Omega_x = \{X=x, Z\}$ and $p(\Omega_x) = \delta_x p(Z)$ where $\delta_x$ is the Dirac measure at $X=x$. For front-door, $\Omega_x =\{X', Z\}$ and $p(\Omega_x) = p(X')p(Z|X=x)$.

\subsection{Conditional Mean Embeddings}
Kernel mean embeddings of distributions provide a powerful framework for representing probability distributions \cite{muandet2017kernel, song2013kernel} in an RKHS. In particular, we work with conditional mean embeddings (\textsc{CME}s) in this paper. Given random variables $X, Y$ with joint distribution $\PP_{XY}$, the conditional mean embedding with respect to the conditional density $p(Y|X=x)$, is defined as:
\begin{align}
    \mu_{Y|X=x}:=\EE_{Y|X=x}[\phi_{y}(Y)] = \int_{\cY}\phi_y(y) p(y|X=x)dy\label{eq: CME1}
\end{align}
\textsc{CME}s allow us to represent the distribution $p(Y|X=x)$ as an element $\mu_{Y|X=x}$ in the RKHS $\cH_{k_y}$ without having to model the densities explicitly. %In cases where we have, for example, multimodality i.e. for a given $x$ there are multiple $y$'s, methods such as GPs or any other standard regression methods will fail to capture the full density, whereas \textsc{CMEs} are able represent densities in an RKHS without any model constraints such as unimodality \citep{pmlr-v130-ton21a}. %Furthermore, this mapping is injective when using a characteristic kernel, i.e. Gaussian kernel \cite{sriperumbudur2011universality}. 
Following \cite{song2013kernel}, \textsc{CMEs} can be associated with a Hilbert-Schmidt operator $\mathcal{C}_{Y|X}:\mathcal H_{k_x} \to \mathcal H_{k_y}$, known as the conditional mean embedding operator, which satisfies $\mu_{Y|X=x} = \mathcal{C}_{Y|X} \phi_x(x)$
% \begin{align}
% \mu_{Y|X=x} = \mathcal{C}_{Y|X} \phi_x(x)
% \label{\textsc}}
% \end{align}
where $\mathcal{C}_{Y|X}:= \mathcal{C}_{YX} \mathcal{C}_{XX}^{-1}$
with $\mathcal{C}_{YX}:=\mathbb{E}_{Y,X}[\phi_y(Y) \otimes \phi_x(X)]$ and $\mathcal{C}_{XX}:=\mathbb{E}_{X,X}[\phi_x(X) \otimes \phi_x(X)]$ being the covariance operators.
As a result, the finite sample estimator of $\mathcal{C}_{Y|X}$ based on the dataset $\{\bfx, \bfy\}$ can be written as:
\begin{equation}
\hat{\mathcal{C}}_{Y|X} = \Phi_\bfy (K_{\bfx\bfx} + \lambda I)^{-1} \Phi_\bfx^T
\label{CMO}
\end{equation}
where $\lambda>0$ is a regularization parameter. Note that from Eq.\ref{CMO}, \cite{grunewalder2012conditional} showed that the \textsc{CME} can be interpret as a vector-valued kernel ridge regressor (V-KRR) i.e. $\phi_x(x)$ is regressed to an element in $\mathcal{H}_{k_y}$. This is crucial as CMEs allow us to turn the integration, in Eq.\ref{eq: CME1}, into a regression task and hence remove the need for explicit density estimation. This insight is important as it allows us to derive analytic forms for our algorithms. Furthermore, the regression formalism of \textsc{CME}s motivated us to derive a Bayesian version of \textsc{CME} using vector-valued Gaussian Processes (V-GP), see Sec.\ref{sec: method}.
\subsection{Interventional Mean Embeddings}
\textit{Interventional Mean Embeddings} (\textsc{IME}) \citep{singh2020kernel} combine the above ideas to represent interventional distributions in RKHSs. We derive the front-door adjustment embedding here but the backdoor adjustment follows analogously. Denote $\mu_{Y|do(X)=x}$ as the \textsc{IME} corresponding to the interventional distribution $p(Y|do(X)=x)$, which can be written as:
\begin{align*}
    \mu_{Y|do(X)=x} &:= \int_{\cY} \phi_y(y)p(y|do(X)=x) dy = \int_{\cX} \int_{\cZ} \underbrace{\Big(\int_{\cY} \phi_y(y) p(y|x', z)dy\Big)}_{\text{\textsc{CME} } \mu_{Y|X=x,Z=z}} p(z|x)p(x')  dz dx' \\
\intertext{using the front-door formula with adjustment variable $Z$, and rearranging the integrals. By definition of \textsc{CME} $\int \phi_y(y)p(y|x', z)dy = C_{Y|X, Z} (\phi_x(x') \otimes  \phi_z(z))$ and linearity of integration, we have}
 &= C_{Y|X,Z} \Big(\underbrace{\int_\cX\phi_x(x')p(x') dx'}_{=\mu_X} \otimes \underbrace{\int_\cZ \phi_z(z) p(z|x) dz}_{=\mu_{Z|X=x}}\Big) = C_{Y|X, Z} \big(\mu_{X} \otimes \mu_{Z|X=x}\big)
\vspace*{-0.3cm}
\end{align*}
Using notations from Sec.\ref{sec:causal}, embedding interventional distributions into an RKHS is as follows. 
\begin{proposition}
\label{proposition: causal_embedding}
Given an identifiable do-density of the form $p(Y|do(X)=x) = \mathbb{E}_{\Omega_x}[p(Y|\Omega_x)]$,  the general form of the empirical interventional mean embedding is given by,
\begin{align}
    \hat{\mu}_{Y|do(X)=x} = \Phi_Y(K_{\Omega_x} + \lambda I)^{-1}\Phi_{\Omega_x}(x)^\top
\end{align}
where $K_{\Omega_x}=K_{XX} \odot K_{ZZ}$ and $\Phi_{\Omega_x}(x)$ is derived depending on $p(\Omega_x)$. In particular, for backdoor adjustments, $\Phi_{\Omega_x}^{(bd)}(x) = \Phi_X^\top k_X(x, \cdot) \odot \Phi_Z^\top \hat{\mu}_{z}$ and for front-door $\Phi_{\Omega_x}^{(fd)}(x) = \Phi_X^\top \hat{\mu}_X \odot \Phi_Z^\top \hat{\mu}_{Z|X=x}$.
\end{proposition}

\vspace{-0.2cm}
\section{Our Proposed Method}\label{sec: method}
\vspace{-0.2cm}

\begin{figure}[!htp]
\begin{minipage}{\textwidth}
\begin{minipage}[htp!]{0.45\textwidth}
\centering
\includegraphics[width=0.93\textwidth]{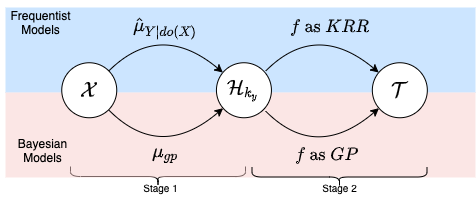}
\captionof{figure}{Two-staged causal learning problem}\label{fig: methods}
\end{minipage}
\hfill
\begin{minipage}[htp!]{0.45\textwidth}
\centering
\begin{tabular}{lcc}
\toprule
METHODS                    &  Stage 1   & Stage 2 \\ 
\midrule
\textsc{IME} \cite{singh2020kernel}     & \textsc{KRR} & \textsc{KRR} \\
IMP (Ours) & \textsc{KRR} & \textsc{GP} \\
\textsc{BayesIME} (Ours)           & \textsc{GP}  & \textsc{KRR} \\
\textsc{BayesIMP} (Ours)            & \textsc{GP}  & \textsc{GP}  \\
\bottomrule
\end{tabular}
\captionof{table}{Summary of our proposed methods}\label{table: methods}
\end{minipage}
\end{minipage}
\vspace*{-0.3cm}
\end{figure}

\paragraph{Two-staged Causal Learning.} Given two independent datasets $\cD_1=\{(x_i, z_i, y_i)\}_{i=1}^N$ and $\cD_2=\{(\tilde{y}_j, t_j)\}_{j=1}^M$, our goal is to model the average treatment effect in $T$ when intervening on variable $X$, i.e model $g(x) = \EE[T|do(X)=x]$. Note that the target variable $T$ and the treatment variable $X$ are never jointly observed. Rather, they are linked via a mediating variable $Y$ observed in both datasets. In our problem setting, we make the following two assumptions: \textbf{(A1)} The treatment only affects the target through the mediating variable, i.e $T \!\perp\!\!\!\perp do(X)| Y$ and \textbf{(A2) }Function $f$ given by $f(y)=\EE[T|Y=y]$ belongs to an RKHS $\cH_{k_y}$.

%The mediating variable satisfies $T \!\perp\!\!\!\perp do(X)| Y$, and
%\textbf{(A2) }Function $f$ given by $f(y)=\EE[T|Y=y]$ belongs to an RKHS $\cH_{k_y}$.
% \begin{enumerate}
%     \item[(A1)] The mediating variable satisfies $T \!\perp\!\!\!\perp do(X)| Y$,
%     \item[(A2)] Function $f$ given by $f(y)=\EE[T|Y=y]$ belongs to an RKHS $\cH_{k_y}$.
% \end{enumerate}

We can thus express the average treatment effect as:
\begin{align}
    g(x)= \EE[T|do(X)=x] &= \int_\mathcal{Y} \underbrace{\EE[T|do(X)=x, Y=y]}_{=\EE[T|Y=y], \text{ since } T \!\perp\!\!\!\perp do(X)| Y} p(y|do(X)=x) dy\\
                   &= \int_\mathcal{Y} f(y) p(y|do(X)=x) dy = \langle f, \mu_{Y|do(X)=x}\rangle_{\cH_{k_y}}.\label{eq: integral}
\end{align}

The final expression decomposes the problem of estimating $g$ into that of estimating the \textsc{IME} $\mu_{Y|do(X)}$ (which can be done using $\cD_1$) and that of estimating the integrand $f:\cY \rightarrow \cT$ (which can be done using $\cD_2$). Each of these two components can either be estimated using a GP or KRR approach (See Table \ref{table: methods}). Furthermore, the reformulation as an RKHS inner product is crucial, as it circumvents the need for density estimation as well as the need for subsequent integration in Eq.\ref{eq: integral}. Rather, the main parts of the task can now be viewed as two instances of regression (recall that mean embeddings can be viewed as vector-valued regression).

To model $g$ and quantify its uncertainty, we propose 3 \textsc{GP}-based approaches. While the first 2 methods, \textit{Interventional Mean Process} (\textsc{IMP}) and \textit{Bayesian Interventional Mean Embedding} (\textsc{BayesIME)} are novel derivations that allow us to quantify uncertainty from either one of the datasets, we treat them as intermediate yet necessary steps to derive our main algorithm, \textit{Bayesian Interventional Mean Process} (\textsc{BayesIMP)}, which allows us to quantify uncertainty from both sources in a principled way. For a summary of the methods, see Fig.\ref{fig: methods} and Table \ref{table: methods}. All derivations are included in the appendix.

%To model $g$ and quantify its uncertainty, we propose the following 3 \textsc{GP}-based approaches. Each of the methods below will be focusing on quantifying the uncertainty from either one or both datasets, hence allowing us to handle cases when datasets come in different quality and quantity. Note that \textit{Interventional Mean Process} (\textsc{IMP}) and \textit{Bayesian Interventional Mean Embedding} (\textsc{BayesIME)} are introduced as necessary intermediate steps for deriving \textsc{BayesIMP}, which is our main algorithm for this paper. For a summary of the methods, see Fig.\ref{fig: methods} and Table \ref{table: methods}. 

\textbf{\emph{Interventional Mean Process}:} Firstly, we train $f$ as a GP using $\cD_2$ and model $\mu_{Y|do(X)=x}$ as \textsc{V-KRR} using $\cD_1$. By drawing parallels to Bayesian quadrature \cite{briol2019probabilistic} and conditional mean process introduced in \cite{chau2021adownscaling}, the integral of interest $g(x) = \int f(y)p(y|do(X)=x)dy$ will be a GP indexed by the treatment variable $X$. We can then use the empirical embedding $\hat{\mu}_{Y|do(X)}$ learnt in $\cD_1$ to obtain an analytic mean and covariance of $g$.

\textbf{\textit{Bayesian Interventional Mean Embedding}}: Next, to account for the uncertainty from $\cD_1$, we model $f$ as a KRR and $\mu_{Y|do(X)=x}$ using a V-\textsc{GP}. We introduce our novel \emph{Bayesian Learning of Conditional Mean Embeddings} (\textsc{BayesCME}), which uses a \emph{nuclear dominant kernel}~\cite{lukic2001stochastic} construction, similar to \cite{flaxman2016bayesian}, to ensure that the inner product $\langle f, \mu_{Y|do(X)=x} \rangle$ is well-defined. As the embedding is a \textsc{GP}, the resulting inner product is also a \textsc{GP} and hence takes into account the uncertainty in $\cD_1$.(See Prop. \ref{proposition: f is krr cme is gp}).

%We will first introduce a novel Bayesian version of \textsc{CME}s, and extend it to interventional embeddings $\mu_{Y|do(X)=x}$ introduced in section \ref{sec: background}. As the embedding is now a \textsc{GP}, the inner product $\langle f, \mu_{Y|do(X)=x} \rangle$ will also be a \textsc{GP} taking into account uncertainty from $\cD_1$. %Modelling the causal embedding with V-\textsc{GP} allows us to take into account the uncertainty from the causal graph $\cD_1$. %Note that estimating the causal CME $\mu_{Y|do(X)=x}$ using a V-\textsc{GP} is novel, and has to the best of our knowledge not been done before. 
     
%Next, we consider the case where we estimate $f$ using a KRR and $\mu_{Y|do(X)=x}$ using V-\textsc{GP}.      
     
%     Viewing $\mu_{Y|do(X)=x}$ using a V-\textsc{GP} allows us to take into account the uncertainty when modelling the causal graph using $\cD_1$. %Here again we note that integration is linear and hence the resulting integral is also Gaussian as shown in \ref{proposition: f is krr cme is gp}.
\textbf{\textit{Bayesian Interventional Mean Process}}: 
Lastly, in order to account for uncertainties coming from both $\cD_1$ and $\cD_2$, we combine ideas from the above $\textsc{IMP}$ and $\textsc{BayesIME}$. We place GPs on both $f$ and $\mu_{Y|do(X)}$ and use their inner product to model $\EE[T|do(X)]$. Interestingly, the resulting uncertainty can be interpreted as the sum of uncertainties coming from \textsc{IMP} and \textsc{BayesIME} with an additional interaction term (See Prop.\ref{proposition: prop: BayesIMP}).

%Note that Chau et al. [add ref] studied this integral in the non-causal setting, i.e. conditional densities instead of \textit{do}-densities, for a very specific deconditioning problem [add ref]. In particular, they introduced the \textsc{Conditional Mean Process} which allows them to place a prior over the conditional means directly, accounting for uncertainty in $f$. In this paper, we adapt their formulation for the causal setting and therefore obtain the \textsc{Causal Conditional Mean Process}.

\subsection{Interventional Mean Process}
% The \emph{Interventional Mean Process} (\textsc{IMP}) is the induced \textsc{GP} we obtain when placing a \textsc{GP} on $f$ and integrating with respect to the interventional density $p(Y|do(X))$. Intuitively, \textsc{IMP} allows us to ``propagate'' uncertainty in the posterior $f$ learnt from $\cD_2$ using the linearity of integration to $g$. The corresponding analytic mean and covariance can then be estimated using the empirical \textsc{IME} $\mu_{Y|do(X)}$ learnt from $\cD_1$.

Firstly, we consider the case where $f$ is modelled using a \textsc{GP} and $\mu_{Y|do(X)=x}$ using a \textsc{V-KRR}. This allows us to take into account the uncertainty from $\cD_2$ by modelling the relationship between $Y$ and $T$ using in a \textsc{GP}. Drawing parallels to Bayesian quadrature \cite{briol2019probabilistic} where integrating $f$ with respect to a marginal measure results into a Gaussian random variable, we integrate $f$ with respect to a conditional measure, thus resulting in a \textsc{GP} indexed by the conditioning variable. Note that \cite{chau2021adownscaling} studied this GP in a non-causal setting, for a very specific downscaling problem. In this work, we extend their approach to model uncertainty in the causal setting. The resulting mean and covariance are then estimated analytically, i.e without integrals, using the empirical \textsc{IME} $\hat\mu_{Y|do(X)}$ learnt from $\cD_1$, see Prop.\ref{proposition: CausalCMP}.

\begin{proposition}[\textbf{\textsc{IMP}}] \label{proposition: CausalCMP}
Given dataset $D_1=\{(x_i, y_i, z_i)\}_{i=1}^N$ and $D_2=\{(\tilde{y}_j, t_j)\}_{j=1}^M$, if $f$ is the posterior \textsc{GP} learnt from $\cD_2$, then $g = \int f(y)p(y|do(X))dy$ is a \textsc{GP} $\cG\cP(m_1, \kappa_1)$ defined on the treatment variable $X$ with the following mean and covariance estimated using $\hat{\mu}_{Y|do(X)}$ ,
\begin{align}
\small
    m_1(x) &= \langle \hat{\mu}_{Y|do(x)}, m_f \rangle_{\cH_{k_y}} = \Phi_{\Omega_x}(x)^\top(K_{\Omega_x} + \lambda I)^{-1} K_{\bfy\tilde{\bfy}}(K_{\tilde{\bfy}\tilde{\bfy}} + \lambda_f I)^{-1}\bft\\
    % m_1(x) &= {\bf t}^\top (K_{\tilde{\bfy}\tilde{\bfy}}+ \lambda_f I)^{-1}K_{\tilde{\bfy}\bfy}(K_{\Omega_x} +  \lambda I)^{-1} \Phi_{\Omega_x}(x) \\
    \kappa_1(x, x') &= \hat{\mu}_{Y|do(x)}^\top \hat{\mu}_{Y|do(x')} - \hat{\mu}_{Y|do(x)}^\top\Phi_{\tilde{\bfy}}(K_{\tilde{\bfy}\tilde{\bfy}} + \lambda I)^{-1}\Phi_{\tilde{\bfy}}^\top  \hat{\mu}_{Y|do(x')} \\
    &= \Phi_{\Omega_x}(x)^\top (K_{\Omega_x} + \lambda I)^{-1} \tilde{K}_{\bfy\bfy} (K_{\Omega_x} + \lambda I)^{-1} \Phi_{\Omega_x}(x')
\end{align}
where $\hat{\mu}_{Y|do(x)} = \hat{\mu}_{Y|do(X)=x}, K_{\tilde{\bfy} \bfy} = \Phi_{\tilde{\bfy}}^\top \Phi_{\bfy}$, $m_f$ and $\tilde{K}_{\bfy\bfy}$ are the posterior mean function and covariance of $f$ evaluated at $\bfy$ respectively. $\lambda > 0$ is the regularisation of the \textsc{CME}. $\lambda_f > 0$ is the noise term for GP $f$. $\Omega_x$ is the set of variables as specified in Prop.\ref{proposition: causal_embedding}. 
\end{proposition}

\textbf{Summary:} The posterior covariance between $x$ and $x'$ in \textsc{IMP} can be interpreted as the similarity between their corresponding empirical \textsc{IMEs} $\hat{\mu}_{Y|do(X)=x}$ and $\hat{\mu}_{Y|do(X)=x'}$ weighted by the posterior covariance $\tilde{K}_{\bfy\bfy}$, where the latter corresponds to the uncertainty when modelling $f$ as a GP in $\cD_2$. However, since $f$ only considers uncertainty in $\cD_2$, we need to develop a method that allows us to quantify uncertainty when learning the \textsc{IME} from $\cD_1$. In the next section, we introduce a Bayesian version of \textsc{CME}, which then lead to $\textsc{BayesIME}$, a remedy to this problem.

\subsection{Bayesian Interventional Mean Embedding}\label{sec:BLCME-KRR}

To account for the uncertainty in $\cD_1$ when estimating $\mu_{Y|do(X)}$, we consider a \textsc{GP} model for \textsc{CME}, and later extend to the interventional embedding \textsc{IME}. We note that Bayesian formulation of CMEs has also been considered in \cite{hsu2018hyperparameter}, but with a specific focus on discrete target spaces.

\textbf{Bayesian learning of conditional mean embeddings with V-GP.} As mentioned in Sec.\ref{sec: background}, CMEs have a clear "feature-to-feature" regression perspective, i.e $\EE[\phi_y(Y)|X=x]$ is the result of regressing $\phi_y(Y)$ onto $\phi_x(X)$. Hence, we consider a vector-valued GP construction to estimate the \textsc{CME}. 

Let $\mu_{gp}(x, y)$ be a \textsc{GP} that models $\mu_{Y|X=x}(y)$. Given that $f \in \mathcal{H}_{k_y}$, for $\langle f, \mu_{gp}(x, \cdot)\rangle_{\cH_{k_y}}$ to be well defined, we need to ensure $\mu_{gp}(x, \cdot)$ is also restricted to $\cH_{k_y}$ for any fixed $x$. Consequently, we cannot define a $\cG\cP(0, k_x\otimes k_y)$ prior on $\mu_{gp}$ as usual, as draws from such prior will almost surely fall outside $\cH_{k_x} \otimes \cH_{k_y}$ \cite{lukic2001stochastic}. Instead we define a prior over $\mu_{gp} \sim \cG\cP(0, k_x \otimes r_y)$, where $r_y$ is a \textit{nuclear dominant kernel}~\cite{lukic2001stochastic} over $k_y$, which ensures that samples paths of $\mu_{gp}$ live in $\cH_{k_x} \otimes \cH_{k_y}$ almost surely. In particular, we follow a similar construction as in \cite{flaxman2016bayesian} and model $r_y$ as $r_y(y_i, y_j) = \int k_y(y_i, u)k_y(u, y_j) \nu(du)$ where $\nu$ is some finite measure on $Y$. Hence we can now setup a vector-valued regression in $\cH_{k_y}$ as follows:
\begin{equation}
    \phi_y(y_i) = \mu_{gp}(x_i, \cdot) + \lambda^{\frac{1}{2}}\epsilon_i
\end{equation}
where $\epsilon_i \sim \cG\cP(0, r)$ are independent noise functions. By taking the inner product with $\phi_y(y')$ on both sides, we then obtain $k_y(y_i, y') = \mu_{gp}(x_i, y') + \lambda^{\frac{1}{2}}\epsilon_i(y')$. Hence, we can treat $k(y_i, y_j)$ as noisy evaluations of $\mu_{gp}(x_i, y_j)$ and obtain the following posterior mean and covariance for $\mu_{gp}$.
% \begin{align}\label{eq: BayesCME model}
%     k_y(y_i, y') = h(x_i, y') + \lambda^{\frac{1}{2}}\epsilon_i(y')
% \end{align}

% As the prior is now well-defined, the observational model is thus, $\phi_y(y_i) = k_{y}(y_i, \cdot) = h(x_i, \cdot) + \lambda^{\frac{1}{2}}\epsilon_i$, where $\epsilon_i \sim \cG\cP(0, r)$ independently across $i$. Note that we are regressing functions on each other. By taking inner products with $\phi_y(y')$, this implies,
% \begin{align}\label{eq: BayesCME model}
%     k_y(y_i, y') = h(x_i, y') + \lambda^{\frac{1}{2}}\epsilon_i(y')
% \end{align}
% The \textsc{GP} posterior of $h$ can thus be inferred by treating $k(y_i, y_j)$ as noisy evaluations of $h(x_i, y_j)$. 

% The resulting mean and covariance for $h$ i.e. the conditional mean embedding $\mu_{Y|X=x}(y)$ can be derived, see Proposition \label{proposition: BL-CME}.

% and hence estimate the inner product $\langle f, \mu_{Y|X=x}\rangle$.

%stacked into $vec(K_{YY})$ 

\begin{proposition}[\textbf{\textsc{BayesCME}}] \label{proposition: BL-CME}

The posterior \textsc{GP} of $\mu_{gp}$ given observations $\{\bfx, \bfy\}$ has the following mean and covariance:
\begin{align}
    m_\mu((x, y)) &= k_{x\bfx}(K_{\bfx\bfx} + \lambda I)^{-1}K_{\bfy\bfy}R_{\bfy\bfy}^{-1}r_{\bfy y} \\
    \kappa_\mu((x, y), (x', y')) &= k_{xx'}r_{y, y'} - k_{x\bfx}(K_{\bfx\bfx} + \lambda I)^{-1}k_{\bfx x'}r_{y\bfy}R_{\bfy\bfy}^{-1}r_{\bfy y'}
\end{align}
In addition, the following marginal likelihood can be used for hyperparameter optimisation,
\begin{equation}
\small
    -\frac{N}{2}\Big(\log|K_{\bfx\bfx} + \lambda I| + \log|R|\Big) - \frac{1}{2}\operatorname{Tr}\Big((K_{\bfx\bfx}+\lambda I)^{-1}K_{\bfy\bfy}R_{\bfy\bfy}^{-1}K_{\bfy\bfy}\Big)
\end{equation}
\end{proposition}
Note that in practice we fix the lengthscale of $k_y$ and $r_y$ when optimising the above likelihood. This is to avoid trivial solutions for the vector-valued regression problem as discussed in \cite{pmlr-v130-ton21a}. The Bayesian version of the \textsc{IME} is derived analogously and we refer the reader to appendix due to limited space.

Finally, with \textsc{V-GPs} on embeddings defined, we can model $g(x)$ as $\langle f, \mu_{gp}(x, \cdot) \rangle_{H_{k_y}}$, which due to the linearity of the inner product, is itself a \textsc{GP}. Here, we first considered the case where $f$ is a \textsc{KRR} learnt from $\cD_2$ and call the model \textsc{BayesIME}.  %Prop. \ref{proposition: f is krr cme is gp} characterises the resulting \textsc{GP} for \textsc{BayesIME}.

%with the above proposition in mind we can now proceed to estimate our target $\langle f, \mu_{Y|do(X)=x}\rangle$. Here, we firstly consider the case where $f$ is a \textsc{KRR} learnt from $\cD_2$ and consider $f$ as a \textsc{GP} in the next section. Proposition \ref{proposition: f is krr cme is gp} contains the mean and covariance for $\langle f, \mu_{Y|do(X)=x}\rangle$ when using a Bayesian conditional mean embedding.

% Note that when $f$ is a \textsc{KRR} learnt from $\cD_2$ we  have that $f(y) = k_{y\tilbfy}(K_{\tilbfy\tilbfy}+\lambda_f I)^{-1}\bft$

\begin{proposition}[\textbf{\textsc{BayesIME}}]
\label{proposition: f is krr cme is gp}
Given dataset $D_1=\{(x_i, y_i, z_i)\}_{i=1}^N$ and $D_2=\{(\tilde{y}_j, t_j)\}_{j=1}^M$, if $f$ is a \textsc{KRR} learnt from $\cD_2$ and $\mu_{Y|do(X)}$ modelled as a \textsc{V-GP} using $\cD_1$, then $g = \langle f, \mu_{Y|do(X)}\rangle \sim \cG\cP(m_2, \kappa_2)$ where,
\begin{align}
    m_2(x) &= \Phi_{\Omega_x}(x)^\top (K_{\Omega_x} + \lambda I)^{-1} K_{\bfy\bfy} R_{\bfy\bfy}^{-1}R_{\bfy\tilbfy} A\\
     \kappa_2(x, x') &= 
     B\Phi_{\Omega_x}(x)^\top \Phi_{\Omega_x}(x) - C\Phi_{\Omega_x}(x)^\top (K_{\Omega_x} + \lambda I)^{-1} \Phi_{\Omega_x}(x')
\end{align}
where $A=(K_{\tilbfy\tilbfy} + \lambda_f I)^{-1}\bft$, $B=A^\top R_{\tilbfy\tilbfy} A$ and $C= A^\top R_{\tilbfy \bfy}R_{\bfy \bfy}^{-1}R_{\bfy \tilbfy}A$
\end{proposition}

\textbf{Summary:} Constants $B$ and $C$ in $\kappa_2$ can be interpreted as different estimation of $||f||_{\cH_{k_y}}$, i.e the RKHS norm of $f$. As a result, for problems that are ``harder'' to learn in $\cD_2$, i.e. corresponding to larger magnitude of $||f||_{\cH_{k_y}}$, will result into larger values of $B$ and $C$. Therefore the covariance $\kappa_2$ can be interpreted as uncertainty in $\cD_1$ scaled by the difficulty of the problem to learn in $\cD_2$.

\subsection{Bayesian Interventional Mean Process}\label{sec: BLCME-GP}
To incorporate both uncertainties in $\cD_1$ and $\cD_2$, we combine ideas from \textsc{IMP} and \textsc{BayesIME} to estimate $g = \langle f, \mu_{Y|do(X)} \rangle$
by placing \textsc{GP}s on both $f$ and $\mu_{Y|do(X)}$. Again as before, a nuclear dominant kernel $r_y$ was used to ensure the \textsc{GP} $f$ is supported on $\cH_{k_y}$. For ease of computation, we consider a finite dimensional approximation of the \textsc{GP}s $f$ and $\mu_{Y|do(X)}$ and estimate $g$ as the RKHS inner product between them. In the following we collate $\bfy$ and $\tilde \bfy$ into a single set of points $\hat \bfy$, which can be seen as landmark points for the finite approximation \cite{trecate1999finite}. We justify this in the Appendix. 

%Lastly, to incorporate both uncertainties in $\cD_1$ and $\cD_2$, we estimate $\langle f, \mu_{Y|do(X)=x} \rangle_{\cH_{k_y}}$ by setting both $f$ and $\mu_{Y|do(X)=x}$ as GPs. Again as in \textsc{BayesIME}, we need to ensure $f$ is supported on $\cH_{k_y}$, hence we place a $\cG\cP(0, r_y)$ prior on $f$. Note that the inner product of two Gaussian random variables are not Gaussian, hence we match the first two moments of the estimation to a \textsc{GP} separately.

%Instead of estimating $\langle f, \mu_{Y|do(X)=x} \rangle_{\cH_{k_y}}$ via two infinite dimensional 

%However, unlike \textsc{IMP} and \textsc{BayesIME}, who are both GPs estimating the $\EE[T|do(X)]$ by taking different uncertainties of the data into account, \textsc{BayesIMP} is estimating the quantity by taking the inner product of two GPs, thus the end quantity is not a \textsc{GP}. However, we can still take the first two moments of the estimation. (ok now modify)

%Note that the resulting quantity from the inner product is in fact not a Gaussian anymore. However, since the target integral is Gaussian itself, we will match the first two moment by estimating the mean and covariance from $\langle f, \mu_{Y|do(X)=x} \rangle$. Below, we derived the first and second moment of the integral of interest.

\begin{proposition}[\textbf{\textsc{BayesIMP}}]\label{proposition: prop: BayesIMP}
Let $f$ and $\mu_{Y|do(X)}$ be \textsc{GP}s learnt as above. Denote $\tilde{f}$ and $\tilde{\mu}_{Y|do(X)}$ as the finite dimensional approximation of $f$ and $\mu_{Y|do(X)}$ respectively. Then $\tilde{g} = \langle \tilde f, \tilde \mu_{Y|do(X)} \rangle$ has the following mean and covariance:
\begin{align}
\small
    m_3(x) &= E_x K_{\bfy\hat{\bfy}}K_{\hat{\bfy}\hat{\bfy}}^{-1}R_{\hat{\bfy}\tilbfy}(R_{\tilbfy\tilbfy} + \lambda_f I)^{-1} \bft \\
    %\kapp_3(x, x') &= \Upsilon_1 + \Upsilon_2 + \Upsilon_3
    \kappa_3(x,x') &= \underbrace{E_x\Theta_1^\top \tilde{R}_{\hat{\bfy}\hat{\bfy}} \Theta_1E_{x'}^\top}_{\text{Uncertainty from } \cD_1 } + \underbrace{\Theta_2^{(a)}F_{xx'} - \Theta_2^{(b)}G_{xx'}}_{\text{Uncertainty from } \cD_2} + \underbrace{\Theta_3^{(a)}F_{xx'} - \Theta_3^{(b)}G_{xx'}}_{\text{Uncertainty from Interaction}}
\end{align}
where $E_x = \Phi_{\Omega_x}(x)^\top (K_{\Omega_x} + \lambda I)^{-1}, F_{xx'}=\Phi_{\Omega_x}(x)^\top \Phi_{\Omega_x}(x'), G_{xx'} = \Phi_{\Omega_x}(x)^\top(K_{\Omega_x} + \lambda I)^{-1}\Phi_{\Omega_x}(x')$, and $\Theta_1 = K_{\hat{\bfy}\hat{\bfy}}^{-1}R_{\hat{\bfy}\bfy}R_{\bfy\bfy}^{-1}K_{\bfy\bfy}$, $\Theta_2^{(a)} = \Theta_4^\top R_{\hat{\bfy}\hat{\bfy}}\Theta_4, \Theta_2^{(b)} =  \Theta_4^\top R_{\hat{\bfy}\bfy}R_{\bfy\bfy}^{-1}R_{\bfy\hat{\bfy}}\Theta_4$ and $\Theta_3^{(a)} = tr(K_{\hat{\bfy} \hat{\bfy}}^{-1}R_{\hat{\bfy}\hat{\bfy}}K_{\hat{\bfy}\hat{\bfy}}^{-1}\bar{R}_{\hat{\bfy}\hat{\bfy}}), \Theta_3^{(b)} = tr(R_{\hat{\bfy}\bfy}R_{\bfy\bfy}^{-1}R_{\bfy\hat{\bfy}}K_{\hat{\bfy}\hat{\bfy}}^{-1}\bar{R}_{\hat{\bfy}\hat{\bfy}}K_{\hat{\bfy}\hat{\bfy}}^{-1})$ and $\Theta_4 = K_{\hat{\bfy}\hat{\bfy}}^{-1}R_{\hat{\bfy}\tilbfy}(K_{\tilbfy\tilbfy}+\lambda_f)^{-1}\bft$. $\bar{R}_{\hat{\bfy}\hat{\bfy}}$ is the posterior covariance of $f$ evaluated at $\hat{\bfy}$ 
\end{proposition}

\textbf{Summary:} While the first two terms in $\kappa_3$ resemble the uncertainty estimates from \textsc{IMP} and \textsc{BayesIME}, the last term acts as an extra interaction between the two uncertainties from $\cD_1$ and $\cD_2$. We note that unlike \textsc{IMP} and \textsc{BayesIME}, $\tilde{g}$ from Prop.\ref{proposition: prop: BayesIMP} is not a \textsc{GP} as inner products between Gaussian vectors are not Gaussian. Nonetheless, the mean and covariance can be estimated.

\section{Experiments}
% \vspace{-0.2cm}

In this section, we first present an ablation studies on how our methods would perform under settings where we have missing data parts at different regions of the two datasets. We then demonstrate \textsc{BayesIMP}'s proficiency in the Causal Bayesian Optimisation setting. 

In particular, we compare our methods against the sampling approach considered in \cite{aglietti2020causal}. \cite{aglietti2020causal} start by modelling $f: Y \rightarrow T$ as GP and estimate the density $p(Y|do(X))$ using a GP along with \textit{do}-calculus. Then given a treatment $x$, we obtain $L$ samples of $y_l$ and $R$ samples of $f_r$ from their posterior GPs. The empirical mean and standard deviation of the samples $\{f_r(y_l)\}_{l=1, r=1}^{L, R}$ can now be taken to estimate $\EE[T|do(X)=x]$ as well as the correspondingly uncertainty. We emphasize that this point estimation requires repeated sampling and is thus inefficient compared to our approaches, where we explicitly model the uncertainty as covariance function.

% The expectation $\EE[T|do(X)=x]$ can then be estimated using $M$ Monte Carlo samples as $\frac{1}{M}\sum_{l=1}^M f(y_l')$, where $y_l'$ are samples from $p(Y|do(X)=x)$. For the second problem, one would model $f$ as a GP, and obtain $M'$ samples of $f_r'$ from the posterior GP. The uncertainty can then be computed by taking the empirical standard deviation of the estimates $\{f'_r(y_l')\}_{l=1, r=1}^{M, M'}$.  

% The first point is often tackled using a two-staged approach. Formally, let $X$ be the treatment, $Y$ be the mediating variable and $T$ our target. \jef{in the previuos example X would have been}To estimate $\EE[T|do(X)]$ we start by computing the regression function $f: Y \rightarrow T$. We can then estimate the density $p(Y|do(X))$ using \textit{do}-calculus \cite{pearl1995causal} and Gaussian Processes (GPs) as in \cite{aglietti2020causal}. The expectation $\EE[T|do(X)=x]$ can then be estimated using $M$ Monte Carlo samples as $\frac{1}{M}\sum_{l=1}^M f(y_l')$, where $y_l'$ are samples from $p(Y|do(X)=x)$. For the second problem, one would model $f$ as a GP, and obtain $M'$ samples of $f_r'$ from the posterior GP. The uncertainty can then be computed by taking the empirical standard deviation of the estimates $\{f'_r(y_l')\}_{l=1, r=1}^{M, M'}$.  

% \paragraph{Sampling-based uncertainty estimation} 

\textbf{Ablation study.} In order to get a better intuition into our methods, we will start off with a preliminary example, where we investigate the uncertainty estimates in a toy case. We assume two simple causal graphs $X \xrightarrow[]{} Y$ for $\cD_1$ and $Y \xrightarrow[]{} T$ for $\cD_2$ and the goal is to estimate $\EE[T|do(X)=x]$ (generating process given in the appendix). We compare our methods from Sec.\ref{sec: method} with the sampling-based uncertainty estimation approach described above. In Fig.\ref{fig: prelim} we plot the mean and the $95\%$ credible interval of the resulting GP models for $\EE[T|do(X)=x]$.  On the $x$-axis we also plotted a histogram of the treatment variable $x$ to illustrate its density.

%We present the sampling approach (Fig. 5(a)), \textsc{IME} \cite{singh2020kernel} (No uncertainty), \textsc{IMP}, \textsc{BayesCME} and \textsc{BayesIME}.

%Note that CBO in addition to the empirical standard deviation also adds a Gaussian kernel to the prior GP to ensure a proper covariance. On the $x$-axis we also plotted a histogram of the input variable $x$ in order to indicate areas with less data.

From Fig.\ref{fig: prelim}(a), we see that the uncertainty for sampling is rather uniform across the ranges of $x$ despite the fact we have more data around $x=0$. This is contrary to our methods, which show a reduction of uncertainty at high $x$ density regions. In particular, $x=-5$ corresponds to an extrapolation of data, where $x$ gets mapped to a region of $y$ where there is no data in $\cD_2$. This fact is nicely captured by the spike of credible interval in Fig.\ref{fig: prelim}(c) since \textsc{IMP} utilises uncertainty from $\cD_2$ directly. Nonetheless, \textsc{IMP} failed to capture the uncertainty stemming from $\cD_1$, as seen from the fact that the credible interval did not increase as we have less data in the region $|x| > 5$. In contrast, \textsc{BayesIME} (Fig.\ref{fig: prelim}(d)) gives higher uncertainty around low $x$ density regions but failed to capture the \textbf{extrapolation} phenomenon. Finally, \textsc{BayesIMP} Fig.\ref{fig: prelim}(e) seems to inherit the desirable characteristics from both $\textsc{IMP}$ and $\textsc{BayesIME}$, due to taking into account uncertainties from both $\cD_1$, $\cD_2$. Hence, in the our experiments, we focus on \textsc{BayesIMP} and refer the reader to the appendix for the remaining methods.

\begin{figure}
    \centering
    \includegraphics[width=\textwidth]{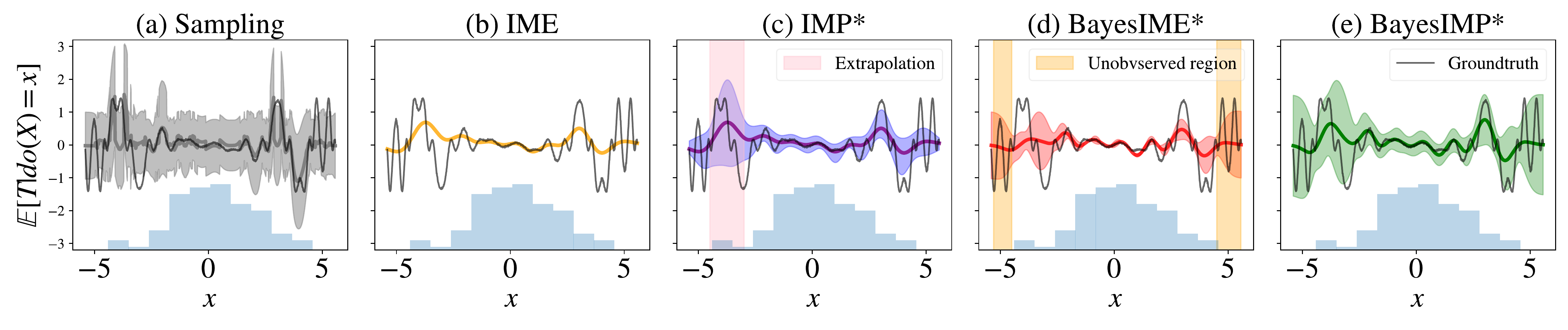}
    \caption{Ablation studies of various methods in estimating uncertainties for an illustrative experiment. $*$ indicates our methods. $N=M=100$ data points are used. Uncertainty from sampling gives a uniform estimate of uncertainty and \textsc{IME} does not come with uncertainty estimates. We see \textsc{IMP} and \textsc{BayesIME} covering different regions of uncertainty while \textsc{BayesIMP} takes the best of both worlds.}
    \label{fig: prelim}
    \vspace*{-0.5cm}
\end{figure}

\textbf{BayesIMP for Bayesian Optimisation (\textsc{BO}).} We now demonstrate, on both synthetic and real-world data, the usefulness of the uncertainty estimates obtained using our methods in BO tasks. Our goal is to utilise the uncertainty estimates to direct the search for the optimal value of $\mathbb{E}[T|do(X)=x]$ by querying as few values of the treatment variable $X$ as possible, i.e. we want to optimize for $x^* = \arg \min_{x\in\mathcal{X}} \EE[T|do(X)=x]$. For the first synthetic experiment (see Fig.\ref{fig: synth1} (Top)), we will use the following two datasets:  $\mathcal{D}_1=\{x_i, u_i, z_i, y_i\}_{i=1}^N$ and  $\mathcal{D}_2=\{\tilde{y}_j, t_j\}_{j=1}^M$. Note that \textsc{BayesIMP} from Prop.\ref{proposition: prop: BayesIMP} is not a \textsc{GP} as inner products between Gaussian vectors are not Gaussian. Nonetheless, with the mean and covariance estimated, we will use moment matching to construct a GP out of \textsc{BayesIMP} for posterior inference. At the start, we are given $\cD_1$ and $\cD_2$, where these observations are used to construct a GP prior for the interventional effect of $X$ on $T$, i.e $\mathbb{E}[T|do(X)=x]$, to provide a ``warm'' start for the \textsc{BO}.

\begin{wrapfigure}{r}{0.31\textwidth}
    \centering
    \includegraphics[width=0.31\textwidth]{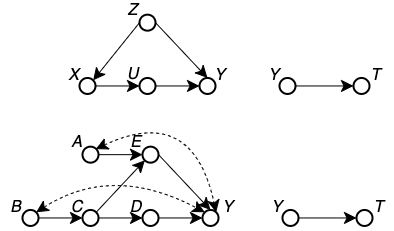}
    \caption{Illustration of synthetic data experiments.}
    \label{fig: synth1}
    \vspace*{-0.3cm}
\end{wrapfigure}

%Our goal is to optimise $\mathbb{E}[T|do(X)=x]$ while using the uncertainty estimate to direct the search 

%Our goal is to optimize the value of $\mathbb{E}[T|do(X)=x]$ by requesting as few as possible interventional data $X$ as possible. At the start, we will be given $N$ observational data points from the causal graph $\cD_{C}$ and $M$ data points from the experimental dataset $\cD_{E}$. This observational data can in turn be used to construct a GP prior for the interventional relationship between $T$ and $X$. This allows us to incorporate observational data into the BO setup, by providing the algorithm a ``warmer'' start. Throughout our experiments, similarly to \cite{aglietti2020causal}, we will be using the expected improvement (EI) acquisition function to select the next point to query.

Again we compare \textsc{BayesIMP} with the sampling-based estimation of $\EE[T|do(X)]$ and its uncertainty, which is exactly the formulation used in the Causal Bayesian Optimisation algorithm (CBO) \cite{aglietti2020causal}. In order to demonstrate how \textsc{BayesIMP} performs in the multimodal setting, we will be considering the case where we have the following distribution on $Y$ i.e. $p(y|u, z) = \pi p_1(y|u, z)  +  (1-\pi) p_2(y|u, z)$ where $Y$ is a mixture and $\pi\in[0,1]$. These scenarios might arise when there is an unobserved binary variable which induces a switching between two regimes on how $Y$ depends on $(U,Z)$. In this case, the GP model of \cite{aglietti2020causal} would only capture the conditional expectation of $Y$ with an inflated variance, leading to slower convergence and higher variance in the estimates of the prior as we will show in our experiments. Throughout our experiments, similarly to \cite{aglietti2020causal}, we will be using the expected improvement (EI) acquisition function to select the next point to query.

\textbf{Synthetic data experiments.} We compare \textsc{BayesIMP} to CBO as well as to a simple GP with no learnt prior as baseline. We will be using $N=100$ datapoints for $\mathcal{D}_1$ and $M=50$ datapoints $\mathcal{D}_2$. We ran each method $10$ times and plot the resulting standard deviation for each iteration in the figures below. The data generation and details were added in the Appendix.
\begin{figure}[!htp]
    \centering
    \includegraphics[width=0.28\textwidth]{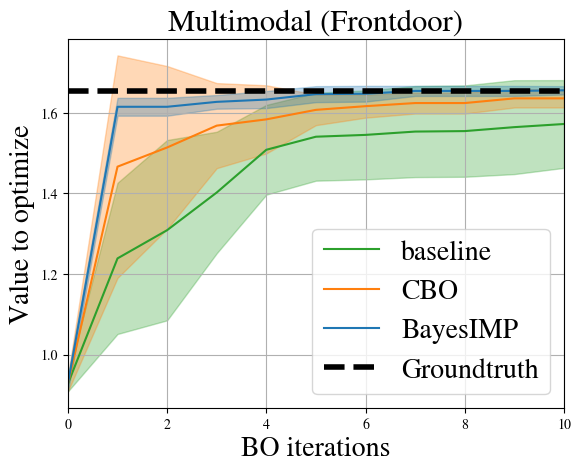}
    \includegraphics[width=0.28\textwidth]{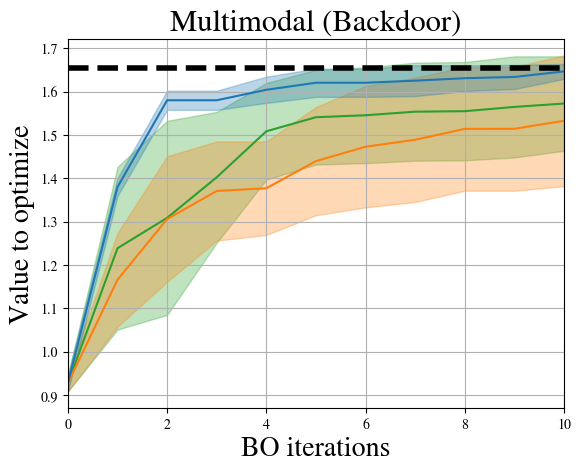}
    \includegraphics[width=0.28\textwidth]{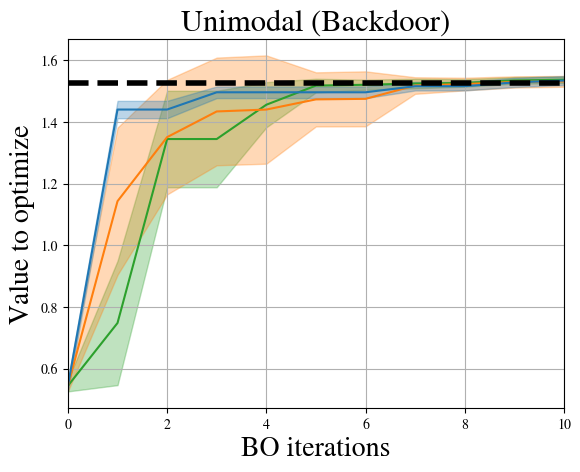}
    \caption{We are interested in finding the maximal value of $\mathbb{E}[T|do(X)=x]$ with as few BO iterations as possible. We ran experiments with \textbf{multimodality} in $Y$. (Left) Using front-door adjustment (Middle) Using backdoor adjustment (Right) Using backdoor adjustment (\textbf{unimodal} $Y$)}
    \label{fig:exp1}
    \vspace{-0.3cm}
\end{figure}
We see from the Fig.\ref{fig:exp1} that \textsc{BayesIMP} is able to find the maxima much faster and with smaller standard deviations, than the current state-of-the-art method, CBO, using both front-door and backdoor adjustments (Fig.\ref{fig:exp1}(Right, Middle)). Given that our method uses more flexible representations of conditional distributions, we are able to circumvent the multimodality problem in $Y$. In addition, we also consider the unimodal version, i.e. $\pi=0$ (see right Fig.\ref{fig:exp1}). We see that the performance of \textsc{CBO} improves in the unimodal setting, however \textsc{BayesIMP} still converges faster than CBO even in this scenario.

% In addition, our proposed methods is also significantly faster than CBO given that we do not require sampling at each optimization step of the EI. This is crucial, as CBO requires recomputation/sampling of the prior at each optimzation step which is significantly slower than using \textsc{BayesIMP} which merely requires a small matrix inversion. 

Next, we consider a harder causal graph (see Fig.\ref{fig: synth1} (Bottom)), previously considered in \cite{aglietti2020causal}. We again introduce multimodality in the $Y$ variable in order to explore the case of more challenging conditional densities. We see from Fig.\ref{fig:exp2} (Left, Middle), that \textsc{BayesIMP} again converges much faster to the true optima than \textsc{CBO} \cite{aglietti2020causal} and the standard GP prior baseline. We note that the fast convergence of \textsc{BayesIMP} throughout our experiments is not due to simplicity of the underlying BO problems. Indeed, the BO with a standard GP prior requires significantly more iterations. It is rather the availability of the observational data, allowing us to construct a more appropriate prior, which leads to a ``warm'' start of the BO procedure.
\begin{figure}[!htp]
    \centering
     \includegraphics[width=0.28\textwidth]{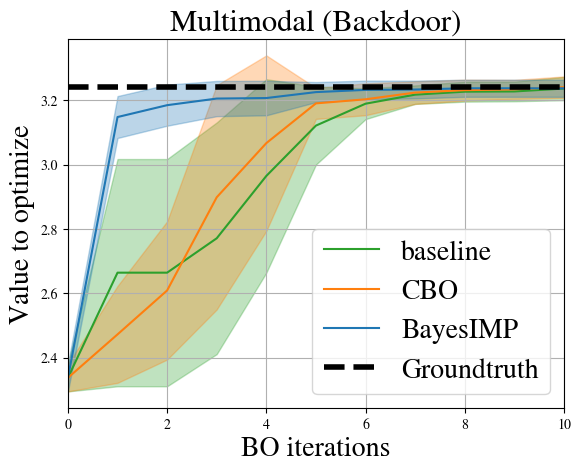}
    \includegraphics[width=0.28\textwidth]{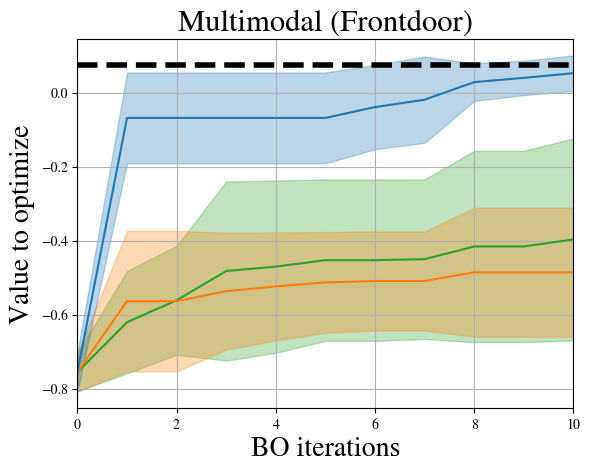}
      \includegraphics[width=0.28\textwidth]{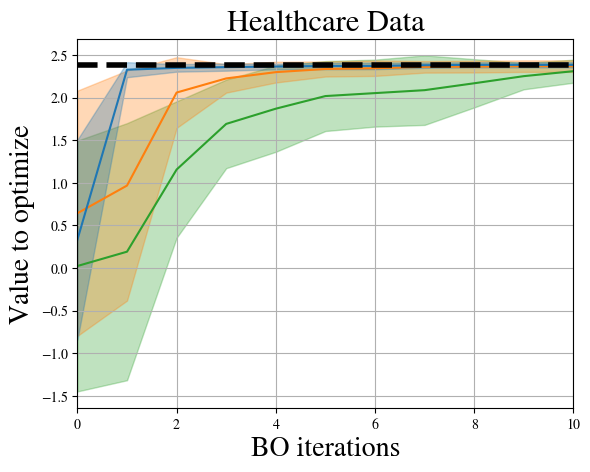}
    \caption{(Left) Experiments where we are interested in $\mathbb{E}[T|do(D)=d]$ with \textbf{multimodal} $Y$, (Middle) Experiments where we are interested in $\mathbb{E}[T|do(E)=e]$ with \textbf{multimodal} $Y$, (Right) Experiments on \textbf{healthcare data} where we are interested in $\EE[\emph{Cancer Volume}|\emph{do(Statin)}]$.}
    \label{fig:exp2}
    \vspace{-0.5cm}
\end{figure}
%\subsection{Real-World Health Care Experiments}
% \begin{figure}[!htp]
%     \centering
%     \includegraphics[width=0.8\textwidth]{99_paper_figs/real_world.png}
%     \caption{Illustration of real-world causal graph from [add ref] and experimental data from [add ref]. The goal in case is to find the best statin dosage to reduce the cancer volume, i.e.  reduce $\EE[\text{Cancer volume}|\text{do(Statin)=statin}]$ \jef{Change the dotted lines and also remove the age from the second graph}}
%     \label{fig:exp_data_setting_harder}
% \end{figure}

\textbf{Healthcare experiments. }We conclude with a healthcare dataset corresponding to our motivating medical example in Fig.\ref{fig:medical example}. The causal mechanism graph, also considered in the CBO paper \cite{aglietti2020causal}, studies the effect of certain drugs (Aspirin/Statin) on Prostate-Specific Antigen (PSA) levels \cite{ferro2015use}. In our case, we modify \emph{statin} to be continuous, in order to optimize for the correct drug dosage. However, in contrast to \cite{aglietti2020causal}, we consider a second experimental dataset, arising from a different medical study, which looks into the connection between \emph{PSA} levels and \emph{cancer volume} amount in patients \cite{stamey1989prostate}. Similar to the original CBO paper \cite{aglietti2020causal}, given that interventional data is hard to obtain, we construct data generators based on the true data collected in \cite{stamey1989prostate}. This is done by firstly fitting a GP on the data and then sampling from the posterior (see Appendix for more details). Hence this is the perfect testbed for our model where we are interested in $\EE[\emph{Cancer Volume}|\emph{do(Statin)}]$. We see from Fig.\ref{fig:exp2} (Right) that \textsc{BayesIMP} again converges to the true optima faster than CBO hence allowing us to find the fastest ways of optimizng \emph{cancer volume} by requesting much less interventional data. This could be critical as interventional data in real-life situations can be very expensive to obtain.
\section{Discussion and Conclusion}
In this paper we propose \textsc{BayesIMP} for quantifying uncertainty in the setting of causal data fusion. In particular, our proposed method \textsc{BayesIMP} allows us to represent interventional densities in the RKHS without explicit density estimation, while still accounting for epistemic and aleatoric uncertainties. We demonstrated the quality of the uncertainty estimates in a variety of Bayesian optimization experiments, in both synthetic and real-world healthcare datasets, and achieve significant improvement over current SOTA in terms of convergence speed. However, we emphasize that \textsc{BayesIMP} is not designed to replace \textsc{CBO} but rather an alternative model for interventional effects. 

%interventional distributions while accounting for uncertainty without explicit density estimation. \textsc{BayesIMP} is a novel tool for capturing interventional distributions into an RKHS from a Bayesian perspective in the causal data fusion setting and we hope to spark additional research interest in the future.
In the future, we would like to improve \textsc{BayesIMP} over several limitations. As in \cite{aglietti2020causal}, we assumed full knowledge of the underlying causal graph, which might be limiting in practice. Furthermore, as the current formulation of \textsc{BayesIMP} only allows combination of two causal graphs, we hope to generalise the algorithm into arbitrary number of graphs in the future. Causal graphs with recurrent structure will be an interesting direction to explore.

\section{Acknowledgements}
The authors would like to thank Bobby He, Robert Hu, Kaspar Martens and Jake Fawkes for helpful comments. SLC and JFT are supported by the EPSRC and MRC through the OxWaSP CDT programme EP/L016710/1. YWT and DS are supported in part by Tencent AI Lab and DS is supported in part by the Alan Turing Institute (EP/N510129/1). YWT’s research leading to these results has received funding from the European Research Council under the European Union’s Seventh Framework Programme (FP7/2007-2013) ERC grant agreement no. 617071.

\bibliographystyle{unsrt}

\newpage
\section*{Checklist}
%%% BEGIN INSTRUCTIONS %%%
% The checklist follows the references.  Please
% read the checklist guidelines carefully for information on how to answer these
% questions.  For each question, change the default \answerTODO{} to \answerYes{},
% \answerNo{}, or \answerNA{}.  You are strongly encouraged to include a {\bf
% justification to your answer}, either by referencing the appropriate section of
% your paper or providing a brief inline description.  For example:
% \begin{itemize}
%   \item Did you include the license to the code and datasets? \answerYes{See Supp material}
% \end{itemize}
% Please do not modify the questions and only use the provided macros for your
% answers.  Note that the Checklist section does not count towards the page
% limit.  In your paper, please delete this instructions block and only keep the
% Checklist section heading above along with the questions/answers below.
%%% END INSTRUCTIONS %%%

\begin{enumerate}
\item For all authors...
\begin{enumerate}
  \item Do the main claims made in the abstract and introduction accurately reflect the paper's contributions and scope?
    \answerYes{}
  \item Did you describe the limitations of your work?
    \answerYes{See section Discussion}
  \item Did you discuss any potential negative societal impacts of your work?
    \answerYes{See below}
  \item Have you read the ethics review guidelines and ensured that your paper conforms to them?
    \answerYes{}
\end{enumerate}
\item If you are including theoretical results...
\begin{enumerate}
  \item Did you state the full set of assumptions of all theoretical results?
    \answerYes{See Appendix}
	\item Did you include complete proofs of all theoretical results?
    \answerYes{See Appendix}
\end{enumerate}
\item If you ran experiments...
\begin{enumerate}
  \item Did you include the code, data, and instructions needed to reproduce the main experimental results (either in the supplemental material or as a URL)?
    \answerYes{See Supp material}
  \item Did you specify all the training details (e.g., data splits, hyperparameters, how they were chosen)?
    \answerYes{See Appendix}
	\item Did you report error bars (e.g., with respect to the random seed after running experiments multiple times)?
    \answerYes{See Experiment section}
	\item Did you include the total amount of compute and the type of resources used (e.g., type of GPUs, internal cluster, or cloud provider)?
    \answerYes{See Appendix}
\end{enumerate}
\item If you are using existing assets (e.g., code, data, models) or curating/releasing new assets...
\begin{enumerate}
  \item If your work uses existing assets, did you cite the creators?
    \answerYes{See Experiment section and Appendix}
  \item Did you mention the license of the assets?
    \answerNA{}
  \item Did you include any new assets either in the supplemental material or as a URL?
    \answerNA{}
  \item Did you discuss whether and how consent was obtained from people whose data you're using/curating?
    \answerNA{}
  \item Did you discuss whether the data you are using/curating contains personally identifiable information or offensive content?
    \answerNA{}
\end{enumerate}
\item If you used crowdsourcing or conducted research with human subjects...
\begin{enumerate}
  \item Did you include the full text of instructions given to participants and screenshots, if applicable?
    \answerNA{}
  \item Did you describe any potential participant risks, with links to Institutional Review Board (IRB) approvals, if applicable?
    \answerNA{}
  \item Did you include the estimated hourly wage paid to participants and the total amount spent on participant compensation?
    \answerNA{}
\end{enumerate}
\end{enumerate}
\section*{Societal Impact}
In this paper we propose a general framework for embedding interventional distributions into a RKHS while accounting for uncertainty in causal datasets. We believe that this is a crucial problem that has not gotten much attention yet, but is nonetheless important for the future of causal inference research. In particular, given that our proposed method allows us to combine datasets from different studies, we envision that this could potentially be used in a variety of scientific areas such as healthcare, drug discovery etc. Finally, the only potential negative impact, would be, when using biased data. Our method relies on knowing the correct causal graph and hence could be misinterpreted when this is not the case.
% \footnote[]{^* \text{Denotes equal contribution with alphabetical ordering}}
% \newpage 

\newpage
% \tableofcontents
% \addtocontents{toc}{\setcounter{tocdepth}{3}} 
\appendix
\section{Additional background on backdoor/front-door adjustments}

In causal inference, we are often times interested in the interventional distributions i.e $p(y|do(x))$ rather than $p(y|x)$, as the former allows us to account for confounding effects. In order to obtain the interventional density $p(y|do(x))$, we resort to \textit{do}-calculus \cite{pearl1995causal}. Here below we write out the definition for the 2 most crucial formulaes; the front-door and backdoor adjustments, with which we are able to recover the interventional density using only the conditional ones.
\subsection{Back-door Adjustment}

The key intuition of back-door adjustments is to find/adjust a set of confounders that are unaffected by the treatment. We can then study the effect of the treatment has to the target.

\begin{definition}[Back-Door]
A set of variables $Z$ satisfies the backdoor criterion relative to an ordered pair of variables $X_i, X_j$ in a DAG $G$ if: 
\begin{enumerate}
    \item no node in $Z$ is a descendant of $X_i$; and
    \item $Z$ blocks every path between $X_i$ and $X_j$ that contains an arrow into $X_i$
\end{enumerate}
Similarly, if $X$ and $Y$ are two disjoint subsets of nodes in $G$, then $Z$ is said satisfy the back-door criterion relative to $(X,Y)$ if it satisfies the criterion relative to any pair $(X_i, X_j)$ such that $X_i \in X$ and $X_j \in Y$
\end{definition}
Now with a given set $Z$ that satisfies the back-door criterion, we apply the backdoor adjustment,
\begin{theorem}[Back-Door Adjustment]
If a set of variables $Z$ satisfies the back-door criterion relative to $(X,Y)$, then the causal effect of $X$ on $Y$ is identifiable and is given by the formula
\begin{align}
    P(y|do(X)=x) = \int_z p(y|x, z)p(z)dz
\end{align}
\end{theorem}

\subsection{Front-door Adjustment}
Front-door adjustment deals with the case where confounders are unobserved and hence the backdoor adjustment is not applicable.

\begin{definition}[Front-door]
A set of variables $Z$ is said to satisfy the front-door criterion relative to an ordered pair of variables $(X,Y)$ if:
\begin{enumerate}
    \item $Z$ intercepts  all directed paths from $X$ to $Y$;
    \item there is no back-door path from $X$ to $Z$; and
    \item all back-door paths from $Z$ to $Y$ are blocked by $X$
\end{enumerate}
\end{definition}
Again, with an appropriate front-door adjustment set $Z$, we can identify the do density using the front-door adjustment formula.
\begin{theorem}[Front-Door Adjustment]
If $Z$ satisfies the front-door criterion relative to $(X,Y)$ and if $P(x,z) > 0$, then the causal effect of $X$ on $Y$ is identifiable and is given by the formula:
\begin{align}
    p(y|do(X)=x) = \int_{z}p(z|x)\int_{x'}p(y|x', z)p(x')dx'dz
\end{align}
\end{theorem}

% \section{Bayesian Learning of Causal Mean Embeddings using V-GP}
% We have derived the Bayesian conditional mean embedding in the main text. Here we note that the causal version is a straight forward extension to the conditional version. Note that the only change is the swap between the canonical feature map $\phi_x$ with $\Phi_{\Omega_x}$.
% \begin{proposition}[\textbf{Bayesian Learning of Causal Conditional Mean Embedding}] \label{proposition: BL-CME}

% The posterior \textsc{GP} of $\mu_{Y|do(X)=x}(y)$ given the above setup has the following mean and covariance:
% \begin{align}
% \smallfont
%     m^{do}(x, y) &= \Phi_{\Omega_x}(\bfx) (K_{\Omega_x} + \lambda I)^{-1} K_{\bfy \bfy} R_{\bfy \bfy}^{-1}r_{\bfy, y} \\ 
%     \kappa^{do}((x, y), (x',y')) &= \Phi_{\Omega_x}(x)^\top \Phi_{\Omega_x}(x') r_{\bfy, y'}  -\Phi_{\Omega_x}(x)^\top (K_{\Omega_x} + \lambda I)^{-1} \Phi_{\Omega_x}(x') r_{y,\bfy}R^{-1}_{\bfy \bfy}r_{\bfy, y'}
% \end{align}
% \end{proposition}

\newpage

\section{Derivations}

\subsection{CMP Derivation}

\begin{customprop}{2}
Given dataset $D_1=\{(x_i, y_i, z_i)\}_{i=1}^N$ and $D_2=\{(\tilde{y}_j, t_j)\}_{j=1}^M$, if $f$ is the posterior \textsc{GP} learnt from $\cD_2$, then $g = \int f(y)p(y|do(X))dy$ is a \textsc{GP} $\cG\cP(m_1, \kappa_1)$ defined on the treatment variable $X$ with the following mean and covariance estimated using $\hat{\mu}_{Y|do(X)}$ ,
\begin{align}
\small
    m_1(x) &= \langle \hat{\mu}_{Y|do(x)}, m_f \rangle_{\cH_{k_y}} = \Phi_{\Omega_x}(x)^\top(K_{\Omega_x} + \lambda I)^{-1} K_{\bfy\tilde{\bfy}}(K_{\tilde{\bfy}\tilde{\bfy}} + \lambda_f I)^{-1}\bft\\
    % m_1(x) &= {\bf t}^\top (K_{\tilde{\bfy}\tilde{\bfy}}+ \lambda_f I)^{-1}K_{\tilde{\bfy}\bfy}(K_{\Omega_x} +  \lambda I)^{-1} \Phi_{\Omega_x}(x) \\
    \kappa_1(x, x') &= \hat{\mu}_{Y|do(x)}^\top \hat{\mu}_{Y|do(x')} - \hat{\mu}_{Y|do(x)}^\top\Phi_{\tilde{\bfy}}(K_{\tilde{\bfy}\tilde{\bfy}} + \lambda I)^{-1}\Phi_{\tilde{\bfy}}^\top  \hat{\mu}_{Y|do(x')} \\
    &= \Phi_{\Omega_x}(x)^\top (K_{\Omega_x} + \lambda I)^{-1} \tilde{K}_{\bfy\bfy} (K_{\Omega_x} + \lambda I)^{-1} \Phi_{\Omega_x}(x')
\end{align}
where $\hat{\mu}_{Y|do(x)} = \hat{\mu}_{Y|do(X)=x}, K_{\tilde{\bfy} \bfy} = \Phi_{\tilde{\bfy}}^\top \Phi_{\bfy}$, $m_f$ and $\tilde{K}_{\bfy\bfy}$ are the posterior mean function and covariance of $f$ evaluated at $\bfy$ respectively. $\lambda > 0$ is the regularisation of the \textsc{CME}. $\lambda_f > 0$ is the noise term for GP $f$. $\Omega_x$ is the set of variables as specified in Prop.\ref{proposition: causal_embedding}. 
\end{customprop}

\begin{proof}[Proof for Proposition \ref{proposition: CausalCMP}]

Integral operator preserves Gaussianity under mild conditions (see conditions \cite{chau2021adownscaling}), therefore
\begin{align}
    g(x) = \int f(y) dP(y|do(X)=x)
\end{align}
is also a Gaussian. For a standard GP prior $f \sim GP(0, k_y)$ and data $D_E = \{(\tilde{y}_j, t_j)\}_{j=1}^M$, standard conjugacy results for GPs lead to the posterior GP with mean $\bar{m}(y) = k_{y \tilde{\bfy}}(K_{\tilde{\bfy}\tilde{\bfy}} + \lambda_f I)^{-1}{\bf t}$ and covariance $\bar{k_y}(y, y') = k_y(y, y') - k_{y\tilde{\bfy}}(K_{\tilde{\bfy}\tilde{\bfy}} + \lambda_f I)^{-1}k_{\tilde{\bfy}y}$. Similar to \cite{briol2019probabilistic}, repeated application of Fubini's theorem yields:
\begin{align}
    \mathbb{E}_f[g(x)] &= \mathbb{E}_f\Bigg[\int f(y) dP(y|do(X)=x\Bigg] = \int \mathbb{E}_f[f(y)]dP(y|do(X)=x)\\
    &= \int \bar{m}(y) dP(y|do(X)=x) = \langle \bar{m}, \hat{\mu}_{Y|do(X)=x}\rangle\\
    cov(g(x), g(x')) &= \int\int cov(f(y), f(y')) dP(y|do(X)=x)dP(y'|do(X)=x) \\
    &= \int\int \bar{k}_y(y, y') dP(y|do(x))dP(y'|do(x')) \\
    &= \langle \mu_{Y|do(x)}, \mu_{Y|do(x')}\rangle - \hat{\mu}_{Y|do(x)}^\top\Phi_{\tilde{\bfy}}(K_{\tilde{\bfy}\tilde{\bfy}} + \lambda I)^{-1}\Phi_{\tilde{\bfy}}^\top  \hat{\mu}_{Y|do(x')} \\
    &= \Phi_{\Omega_x}(x)^\top (K_{\Omega_x} + \lambda I)^{-1} \tilde{K}_{\bfy\bfy} (K_{\Omega_x} + \lambda I)^{-1} \Phi_{\Omega_x}(x')
\end{align}
\end{proof}

\subsection{Choice of Nuclear Dominant Kernel} 

Recall in section \ref{sec:BLCME-KRR}, we introduced the nuclear dominant kernel $r_y$ to ensure samples of $\mu_{gp} \sim GP(0, k_x\otimes r_y)$ are supported in $\cH_{k_x} \otimes \cH_{k_y}$ with probability 1. In the following we will present the analytic form of the nuclear dominant kernel we used in this paper, which is the same as the formulation introduced in Appendix A.2 and A.3 of \cite{flaxman2016bayesian}. Pick $k_y$ as the RBF kernel, i.e 
\begin{align}
    k_y(y,y') = \exp\Big(-\frac{1}{2}(y-y')^\top \Sigma_\theta (y-y')\Big)
\end{align}
where $\Sigma_\theta$ is covariance matrix for the kernel $k_y$. The nuclear dominant kernel construction from \cite{flaxman2016bayesian} then yield the following expression:
\begin{align}
    r_y(y, y') = \int k_y(y, u) k_y(u, y') \nu(du)
\end{align}
where $\nu$ is some finite measure. If we pick $\nu(du) = \exp(\frac{||u||_2^2}{2\eta^2}) du$, then we have 
\begin{align}
    r_y(y, y') &= (2\pi)^{D/2}|2\Sigma_{\theta}^{-1} + \eta^{-2} I|^{-1/2}\exp \Big(-\frac{1}{2} (y-y')^\top (2\Sigma_{\theta})^{-1}(y-y')\Big) \\
    &\quad\quad\quad \times \exp\Big(-\frac{1}{2}\Big(\frac{y+y'}{2}\Big)^\top \Big(\frac{1}{2}\Sigma_\theta + \eta^2 I\Big)^{-1}\Big(\frac{y+y'}{2}\Big)\Big)
\end{align}

\subsection{BayesCME derivations}

\begin{customprop}{3}
The posterior \textsc{GP} of $\mu_{gp}$ given observations $\{\bfx, \bfy\}$ has the following mean and covariance:
\begin{align}
    m_\mu((x, y)) &= k_{x\bfx}(K_{\bfx\bfx} + \lambda I)^{-1}K_{\bfy\bfy}R_{\bfy\bfy}^{-1}r_{\bfy y} \\
    \kappa_\mu((x, y), (x', y')) &= k_{xx'}r_{y, y'} - k_{x\bfx}(K_{\bfx\bfx} + \lambda I)^{-1}k_{\bfx x'}r_{y\bfy}R_{\bfy\bfy}^{-1}r_{\bfy y'}
\end{align}
In addition, the following marginal likelihood can be used for hyperparameter optimisation,
\begin{equation}
\small
    -\frac{N}{2}\Big(\log|K_{\bfx\bfx} + \lambda I| + \log|R|\Big) - \frac{1}{2}\operatorname{tr}\Big((K_{\bfx\bfx}+\lambda I)^{-1}K_{\bfy\bfy}R_{\bfy\bfy}^{-1}K_{\bfy\bfy}\Big)
\end{equation}
\end{customprop}

\begin{proof}[Proof of Proposition 3.]

Recall the Bayesian formulation of CME corresponds to the following model,
\begin{eqnarray*}
\mu_{gp} & \sim & GP\left(0,k_x\otimes r_y\right),\\
k_y\left(y_{i},y'\right) & = & \mu_{gp}\left(x_{i},y'\right)+\lambda^{1/2}\epsilon_{i}\left(y'\right)
\end{eqnarray*}
with $\epsilon_{i}\sim GP(0,r_y)$ independently across $i$. Now consider $k_y(y_i, y_j)$ as noisy evaluations of $\mu_{gp}(x_i, y_j)$, we have the predictive posterior mean as

\begin{eqnarray*}
\textrm{vec}\left(r_{{\bf y}y}k_{x{\bf x}}\right)^{\top}\left(K_{\bfx\bfx}\otimes R_{\bfy\bfy}+\lambda I\otimes R_{\bfy\bfy}\right)^{-1}\textrm{vec}(K_{\bfy\bfy}) & = & \textrm{vec}\left(r_{{\bf y}y}k_{x{\bf x}}\right)^{\top}\left(\left(K_{\bfx\bfx}+\lambda I\right)^{-1}\otimes R_{\bfy\bfy}^{-1}\right)\textrm{vec\ensuremath{\left(K_{\bfy\bfy}\right)}}\\
 & = & \textrm{vec}\left(r_{{\bf y}y}k_{x{\bf x}}\right)^{\top}\textrm{vec}\left(R_{\bfy\bfy}^{-1}K_{\bfy\bfy}\left(K_{\bfx\bfx}+\lambda I\right)^{-1}\right)\\
 & = & \textrm{tr}\left(r_{{\bf y}y}k_{x{\bf x}}\left(K_{\bfx\bfx}+\lambda I\right)^{-1}K_{\bfy\bfy}R_{\bfy\bfy}^{-1}\right)\\
 & = & k_{x{\bf x}}\left(K_{\bfx\bfx}+\lambda I\right)^{-1}K_{\bfy\bfy}R_{\bfy\bfy}^{-1}r_{{\bf y}y}.
\end{eqnarray*}

And the covariance is,

\begin{eqnarray*}
\kappa\left(\left(x,y\right),\left(x',y'\right)\right) & = & k(x,x')r(y,y')-\textrm{vec}\left(r_{{\bf y}y}k_{x{\bf x}}\right)^{\top}\left(K_{\bfx\bfx}\otimes R_{\bfy\bfy}+\lambda I\otimes R_{\bfy\bfy}\right)^{-1}\textrm{vec}\left(r_{{\bf y}y'}k_{x'{\bf x}}\right)\\
 & = & k(x,x')r(y,y')-\textrm{vec}\left(r_{{\bf y}y}k_{x{\bf x}}\right)^{\top}\left(\left(K_{\bfx\bfx}+\lambda I\right)^{-1}\otimes R_{\bfy\bfy}^{-1}\right)\textrm{vec}\left(r_{{\bf y}y'}k_{x'{\bf x}}\right)\\
 & = & k(x,x')r(y,y')-\textrm{vec}\left(r_{{\bf y}y}k_{x{\bf x}}\right)^{\top}\textrm{vec}\left(R_{\bfy\bfy}^{-1}r_{{\bf y}y'}k_{x'{\bf x}}\left(K_{\bfx\bfx}+\lambda I\right)^{-1}\right)\\
 & = & k(x,x')r(y,y')-\textrm{tr}\left(r_{{\bf y}y}k_{x{\bf x}}\left(K_{\bfx\bfx}+\lambda I\right)^{-1}k_{{\bf x}x'}r_{y'{\bf y}}R_{\bfy\bfy}^{-1}\right)\\
 & = & k(x,x')r(y,y')-k_{x{\bf x}}\left(K_{\bfx\bfx}+\lambda I\right)^{-1}k_{{\bf x}x'}r_{y'{\bf y}}R_{\bfy\bfy}^{-1}r_{{\bf y}y}.
\end{eqnarray*}

To compute the log likelihood, note that it contains the following two terms:
\begin{eqnarray*}
\textrm{vec}(K_{\bfy\bfy})^\top \left(K_{\bfx\bfx}\otimes R_{\bfy\bfy}+\lambda I\otimes R_{\bfy\bfy}\right)^{-1}\textrm{vec}(K_{\bfy\bfy}) & = & \textrm{vec}(K_{\bfy\bfy})^{\top}\left(\left(K_{\bfx\bfx}+\lambda I\right)^{-1}\otimes R_{\bfy\bfy}^{-1}\right)\textrm{vec\ensuremath{\left(K_{\bfy\bfy}\right)}}\\
& = & \textrm{vec}(K_{\bfy\bfy})^{\top}\textrm{vec}\left(R_{\bfy\bfy}^{-1}K_{\bfy\bfy}\left(K_{\bfx\bfx}+\lambda I\right)^{-1}\right)\\
& = & \textrm{tr}\left(K_{\bfy\bfy}\left(K_{\bfx\bfx}+\lambda I\right)^{-1}K_{\bfy\bfy}R_{\bfy\bfy}^{-1}\right)\\
\end{eqnarray*}
and 
\begin{eqnarray*}
-\frac{1}{2}\Big(\log |(K_{\bfx\bfx}+\lambda I)\otimes R_{\bfy\bfy} |\Big) & = & -\frac{1}{2} \log\Big(|(K_{\bfx\bfx}+\lambda I)|^N |R|^N\Big) \\
& = & -\frac{N}{2}\Big(\log|K_{\bfx\bfx} + \lambda I| + \log|R|\Big)
\end{eqnarray*}
where we used the fact that determinant of Kronecker product of two $N\times N$ matrices $A, B$ is: $|A\otimes B| = |A|^N|B|^N$.

Therefore the log likelihood can be expressed as
\begin{equation}
    -\frac{N}{2}\Big(\log|K_{\bfx\bfx} + \lambda I| + \log|R|\Big) - \frac{1}{2}\operatorname{tr}\Big((K_{\bfx\bfx}+\lambda I)^{-1}K_{\bfy\bfy}R_{\bfy\bfy}^{-1}K_{\bfy\bfy}\Big)    
\end{equation}

\end{proof}

\subsection{Causal BayesCME derivations}
The following proposition extend \textsc{BayesCME} to the causal setting.

\begin{customprop}{C.1}[Causal BayesCME]
Denote $\mu_{gp}^{do}$ as the GP modelling $\mu_{Y|do(X)}$. Then using the $\Omega$ notations introduced in proposition \ref{proposition: causal_embedding}, the posterior \textsc{GP} of $\mu_{gp}^{do}$ given observations $\{\bfx,\bfz, \bfy\}$ has the following mean and covariance:
\begin{align}
% \smallfont
    m_\mu^{do}((x, y)) &= \Phi_{\Omega_x}(x)^\top \Big(K_{\Omega_x} + \lambda I\Big)^{-1}K_{\bfy\bfy}R_{\bfy\bfy}^{-1}r_{\bfy y}\\
    \kappa_\mu^{do}((x, y), (x', y')) &= \Phi_{\Omega_x}(x)^\top\Phi_{\Omega_x}(x')r_{y, y'} - \Phi_{\Omega_x}(x)^\top(K_{\Omega_x} + \lambda I)^{-1}\Phi_{\Omega_x}(x')r_{y\bfy}R_{\bfy\bfy}^{-1}r_{\bfy y'}
\end{align}
In addition, the following marginal likelihood can be used for hyperparameter optimisation,
\begin{equation}
\small
    -\frac{N}{2}\Big(\log|K_{\Omega_x} + \lambda I| + \log|R|\Big) - \frac{1}{2}\operatorname{tr}\Big((K_{\Omega_x}+\lambda I)^{-1}K_{\bfy\bfy}R_{\bfy\bfy}^{-1}K_{\bfy\bfy}\Big)
\end{equation}
\end{customprop}
\begin{proof}[Proof of Proposition C.1]
In the following we will assume $Z$ is the backdoor adjustment variable. Front-door and general cases follow analogously. Denote $\mu_{gp}((x,z), y)$ as the \textsc{BayesCME} model for $\mu_{Y|X=x, Z=z}(y)$. As we have
\begin{align}
    \mu_{Y|do(X)=x} &= \int \int \phi_y(y) p(y|x,z)p(z)dz dy \\
                    &= \int \mu_{Y|X=x, Z=z} p(z)dz \\
                    &= \EE_Z[\mu_{Y|X=x,Z}]
\end{align}
It is thus natural to define $\mu_{gp}^{do}$ as the induced GP when we replace $\mu_{Y|X=x, Z=z}$ with $\mu_{gp}((x,z), \cdot)$,
\begin{align}
\mu_{gp}^{do}(x, \cdot) = \EE_Z[\mu_{gp}((x,Z), \cdot)]    
\end{align}
Now we can compute the mean of $\mu_{gp}^{do}$,
\begin{align}
m_{\mu}^{do}(x, y) &= \EE_{\mu_{gp}}\EE_Z[\mu_{gp}(x,Z,y)] \\
                   &= \EE_Z\Big((k_{x\bfx}\odot k_z(Z,\bfz)) (K_{\bfx\bfx} \odot K_{\bfz\bfz} + \lambda I)^{-1}K_{\bfy\bfy}R_{\bfy\bfy}^{-1}r_{\bfy y}\Big) \\
                   &= \Big((k_{x\bfx}\odot \mu_z^\top \Phi_\bfz) (K_{\bfx\bfx} \odot K_{\bfz\bfz} + \lambda I)^{-1}K_{\bfy\bfy}R_{\bfy\bfy}^{-1}r_{\bfy y}\Big) \\
                   &= \Phi_{\Omega_x}(x)^\top (K_{\Omega_x} + \lambda I)^{-1} K_{\bfy\bfy}R_{\bfy\bfy}^{-1}r_{\bfy y}
\end{align}
Similarly for covariance, we have,
\begin{align}
    \kappa_{\mu}^{do}((x, y), (x',y')) &= \EE_{Z,Z'}[cov\big(\mu_{gp}((x,Z),y), \mu_{gp}((x', Z'), y')\big)] \\
\intertext{and the rest is just algebra,}
&= \Phi_{\Omega_x}(x)^\top\Phi_{\Omega_x}(x')r_{y, y'} - \Phi_{\Omega_x}(x)^\top(K_{\Omega_x} + \lambda I)^{-1}\Phi_{\Omega_x}(x')r_{y\bfy}R_{\bfy\bfy}^{-1}r_{\bfy y'}
\end{align}
\end{proof}

\subsection{BayesIME derivation}
Now we have derived the Causal \textsc{BayesCME}, it is time to compute $\langle f, \mu_{gp}^{do}(x,\cdot)\rangle$ where $f\in\cH_{k_y}$. This requires us to be able to compute $\langle f, r_y(\cdot, y) \rangle$ which corresponds to the following:
\begin{align}
    \langle f, r_y(\cdot, y) \rangle_{\cH_{k_y}} &= \Big\langle f, \int k_y(\cdot, u)k_y(u, y) \nu(du)\Big\rangle \\
    &= \int f(u)k_y(u, y)\nu(du) \\
\intertext{when $f$ is a \textsc{KRR} learnt from $\cD_2$, i.e $f(y) = k_{y\tilbfy}(K_{\tilbfy\tilbfy}+\lambda_f I)^{-1}\bft$, we have}
    &= \bft^\top (K_{\tilbfy\tilbfy} + \lambda_f I)^{-1}\int k_{\tilbfy u} k_y(u, y) \nu(du) \\
    &= \bft^\top(K_{\tilbfy\tilbfy} + \lambda_f I)^{-1}r_{\tilbfy y}
\end{align}
Now we are ready to derive \textsc{BayesIME}.

\begin{customprop}{\ref{proposition: f is krr cme is gp}}
Given dataset $D_1=\{(x_i, y_i, z_i)\}_{i=1}^N$ and $D_2=\{(\tilde{y}_j, t_j)\}_{j=1}^M$, if $f$ is a \textsc{KRR} learnt from $\cD_2$ and $\mu_{Y|do(X)}$ modelled as a \textsc{V-GP} using $\cD_1$, then $g = \langle f, \mu_{Y|do(X)}\rangle \sim \cG\cP(m_2, \kappa_2)$ where,
\begin{align}
    m_2(x) &= \Phi_{\Omega_x}(x)^\top (K_{\Omega_x} + \lambda I)^{-1} K_{\bfy\bfy} R_{\bfy\bfy}^{-1}R_{\bfy\tilbfy} A\\
     \kappa_2(x, x') &= 
     B\Phi_{\Omega_x}(x)^\top \Phi_{\Omega_x}(x) - C\Phi_{\Omega_x}(x)^\top (K_{\Omega_x} + \lambda I)^{-1} \Phi_{\Omega_x}(x')
\end{align}
where $A=(K_{\tilbfy\tilbfy} + \lambda_f I)^{-1}\bft$, $B=A^\top R_{\tilbfy\tilbfy} A$ and $C= A^\top R_{\tilbfy \bfy}R_{\bfy \bfy}^{-1}R_{\bfy \tilbfy}A$
\end{customprop}

\begin{proof}[Proof of Proposition \ref{proposition: f is krr cme is gp}] Using the $\mu_{gp}^{do}$ notation from Proposition $C.1$, we can write the inner product as $\langle \mu_{gp}^{do}(x, \cdot), f\rangle$, where the mean is,
\begin{align}
    m_2(x) &= \EE[\mu_{gp}^{do}(x, \cdot)]^\top f \\
           &= \Phi_{\Omega_x}(x)^\top (K_{\Omega_x} + \lambda I)^{-1} K_{\bfy\bfy} R_{\bfy\bfy}^{-1}R(\bfy, \cdot)^\top f \\
           &= \Phi_{\Omega_x}(x)^\top (K_{\Omega_x} + \lambda I)^{-1} K_{\bfy\bfy} R_{\bfy\bfy}^{-1}R_{\bfy\tilbfy}(K_{\tilbfy\tilbfy} + \lambda_f I)^{-1}\bft
\end{align}
where we used the fact $f$ is a KRR learnt from $\cD_2$. The covariance can then be computed by realising $cov( f^\top \mu_{gp}^{do}(x, \cdot), f^\top \mu_{gp}^{do}(x', \cdot)) = f^\top cov(\mu_{gp}^{do}(x, \cdot), \mu_{gp}^{do}(x', \cdot))f$.
\end{proof}

\subsection{BayesIMP Derivations}

\textsc{BayesIMP} can be understood as a model characterising the RKHS inner product of Gaussian Processes. In the following, we will first introduce some general theory of inner product of GPs, and introduce a finite dimensional scheme later on. Finally, we will show how \textsc{BayesIMP} can be derived right away from this general framework.

Before that, we will showcase the following identity for computing variance of inner products of independent multivariate Gaussians,

\begin{customprop}{C.2}
Let $\mu_X:=\mathbb{E}[X]$ and $\Sigma_X := Var(X)$ be the mean and variance of a multivariate Gaussian rv, similarly $\mu_Y, \Sigma_Y$ for Gaussian rv $Y$. If $X$ and $Y$ are independent, then the variance of their inner product is given by the following expression,
\begin{align}
    Var(X^\top Y) = \mu_X^\top \Sigma_Y \mu_X + \mu_Y^\top \Sigma_X \mu_Y + tr\Big({\Sigma_Y \Sigma_X}\Big)
\end{align}
Moreover, the covariance between $X^\top Y_1$, $X^\top Y_2$ follows a similar form,
\begin{align}
    cov(X^\top Y_1, X^\top Y_2) &= \mu_X^\top \Sigma_{Y_1Y_2}\mu_X + \mu_{Y_1}^\top \Sigma_X \mu_{Y_2} + \operatorname{tr}(\Sigma_X\Sigma_{Y_1Y_2})
\end{align}
\end{customprop}

\begin{proof}
\begin{equation}
\begin{aligned}
\operatorname{Var}\left[X^{\top} Y\right] &=\mathbb{E}\left[\left(X^{\top} Y\right)^{2}\right]-\mathbb{E}\left[X^{\top} Y\right]^{2} \\
&=\mathbb{E}\left[X^{\top} Y Y^{\top} X\right]-\left(\mathbb{E}[X]^{\top} \mathbb{E}[Y]\right)^{2} \\
&=\mathbb{E}\left[\operatorname{tr}\left(X X^{\top} Y Y^{\top}\right)\right]-\left(\mu_{X}^{\top} \mu_{Y}\right)^{2} \\
&=\operatorname{tr}\left(\mathbb{E}\left[X X^{\top}\right] \mathbb{E}\left[Y Y^{\top}\right]\right)-\left(\mu_{X}^{\top} \mu_{Y}\right)^{2} \\
&=\operatorname{tr}\left(\left(\mu_{X} \mu_{X}^{\top}+\Sigma_{X}\right)\left(\mu_{Y} \mu_{Y}^{\top}+\Sigma_{Y}\right)\right)-\left(\mu_{X}^{\top} \mu_{Y}\right)^{2} \\
&=\operatorname{tr}\left(\mu_{X} \mu_{X}^{\top} \mu_{Y} \mu_{Y}^{\top}\right)+\operatorname{tr}\left(\mu_{X} \mu_{X}^{\top} \Sigma_{Y}\right)+\operatorname{tr}\left(\Sigma_{X} \mu_{Y} \mu_{Y}^{\top}\right)+\operatorname{tr}\left(\Sigma_{X} \Sigma_{Y}\right)-\left(\mu_{X}^{\top} \mu_{Y}\right)^{2} \\
&=\left(\mu_{X}^{\top} \mu_{Y}\right)^{2}+\operatorname{tr}\left(\mu_{X}^{\top} \Sigma_{Y} \mu_{X}\right)+\operatorname{tr}\left(\mu_{Y}^{\top} \Sigma_{X} \mu_{Y}\right)+\operatorname{tr}\left(\Sigma_{X} \Sigma_{Y}\right)-\left(\mu_{X}^{\top} \mu_{Y}\right)^{2} \\
&=\mu_{X}^{\top} \Sigma_{Y} \mu_{X}+\mu_{Y}^{\top} \Sigma_{X} \mu_{Y}+\operatorname{tr}\left(\Sigma_{X} \Sigma_{Y}\right)
\end{aligned}
\end{equation}
Generalising to the case for covariance is straight forward.
\end{proof}

\subsubsection*{RKHS inner product of Gaussian Processes}
Let $f_{1}\sim GP(m_{1},\kappa_{1})$ and $f_{2}\sim GP(m_{2},\kappa_{2})$.
We assume that $f$ and $g$ are both supported within the RKHS $\mathcal{H}_{k}$.
Can we characterise the distribution of $\left\langle f_{1},f_{2}\right\rangle _{\mathcal{H}_{k}}$? 

This situation would arise if $f_{1}$ and $f_{2}$ arise as GP posteriors
in a regression model corresponding to the priors $f_{1}\sim GP(0,r_{1})$,
$f_{2}\sim GP(0,r_{2})$ where $r_{1},r_{2}$ satisfy the nuclear
dominance property. In particular, we could choose
\begin{eqnarray*}
r_{1}(u,v) & = & \int k\left(u,z\right)k\left(z,v\right)\nu_{1}\left(dz\right),\\
r_{2}\left(u,v\right) & = & \int k\left(u,z\right)k\left(z,v\right)\nu_{2}\left(dz\right).
\end{eqnarray*}

Posterior means in that case can be expanded as 
\[
m_{1}=\sum\alpha_{i}r_{1}\left(\cdot,x_{i}\right),\qquad m_{2}=\sum\beta_{j}r_{2}\left(\cdot,y_{j}\right).
\]

We assume that $f_{1}$ and $f_{2}$ are independent, i.e. they correspond
to posteriors computed on independent data. Then
\begin{eqnarray*}
\mathbb{E}\left\langle f_{1},f_{2}\right\rangle _{\mathcal{H}_{k}} & = & \left\langle m_{1},m_{2}\right\rangle _{\mathcal{H}_{k}}\\
 & = & \left\langle \sum\alpha_{i}r_{1}\left(\cdot,x_{i}\right),\sum\beta_{j}r_{2}\left(\cdot,y_{j}\right)\right\rangle _{\mathcal{H}_{k}}\\
 & = & \alpha^{\top}Q\beta,
\end{eqnarray*}
where
\begin{eqnarray*}
Q_{ij}=q\left(x_{i},y_{j}\right) & := & \left\langle r_{1}\left(\cdot,x_{i}\right),r_{2}\left(\cdot,y_{j}\right)\right\rangle _{\mathcal{H}_{k}}\\
 & = & \left\langle \int k\left(\cdot,z\right)k\left(z,x_{i}\right)\nu_{1}\left(dz\right),\int k\left(\cdot,z'\right)k\left(z',y_{j}\right)\nu_{2}\left(dz'\right)\right\rangle _{\mathcal{H}_{k}}\\
 & = & \int\int\left\langle k\left(\cdot,z\right),k\left(\cdot,z'\right)\right\rangle _{\mathcal{H}_{k}}k\left(z,x_{i}\right)k\left(z',y_{j}\right)\nu_{1}\left(dz\right)\nu_{2}\left(dz'\right)\\
 & = & \int\int k\left(z,z'\right)k\left(z,x_{i}\right)k\left(z',y_{j}\right)\nu_{1}\left(dz\right)\nu_{2}\left(dz'\right).
\end{eqnarray*}

The variance would be given, in analogy to the
finite dimensional case, by
\[
\text{var}\left\langle f_{1},f_{2}\right\rangle _{\mathcal{H}_{k}}=\left\langle m_{1},\Sigma_{2}m_{1}\right\rangle _{\mathcal{H}_{k}}+\left\langle m_{2},\Sigma_{1}m_{2}\right\rangle _{\mathcal{H}_{k}}+\text{tr}\left(\Sigma_{1}\Sigma_{2}\right),
\]

with $\Sigma_{1}f=\int\kappa_{1}\left(\cdot,u\right)f(u)du$ and similarly
for $\Sigma_{2}$. Thus
\begin{eqnarray*}
\left\langle m_{1},\Sigma_{2}m_{1}\right\rangle _{\mathcal{H}_{k}} & = & \left\langle \sum\alpha_{i}r_{1}\left(\cdot,x_{i}\right),\sum\alpha_{j}\int\kappa_{2}\left(\cdot,u\right)r_{1}\left(u,x_{j}\right)du\right\rangle _{\mathcal{H}_{k}}\\
 & = & \sum\sum\alpha_{i}\alpha_{j}\int\left\langle r_{1}\left(\cdot,x_{i}\right),\kappa_{2}\left(\cdot,u\right)\right\rangle _{\mathcal{H}_{k}}r_{1}\left(u,x_{j}\right)du.
\end{eqnarray*}

Now, given that kernel $\kappa_{2}$ depends on $r_{2}$ in a simple
way, it should be possible to write down the full expression similarly as for $Q_{ij}$ above. In particular 
% \jef{How? i am confused here}

\[
\kappa_{2}\left(\cdot,u\right)=r_{2}\left(\cdot,u\right)-r_{2}\left(\cdot,{\bf y}\right)\left(R_{2,{\bf y{\bf y}}}+\sigma_{2}^{2}I\right)^{-1}r_{2}\left({\bf y},u\right).
\]

Hence
\begin{eqnarray*}
\left\langle r_{1}\left(\cdot,x_{i}\right),\kappa_{2}\left(\cdot,u\right)\right\rangle _{\mathcal{H}_{k}} & = & q\left(x_{i},u\right)-q\left(x_{i},{\bf y}\right)\left(R_{2,{\bf y{\bf y}}}+\sigma_{2}^{2}I\right)^{-1}r_{2}\left({\bf y},u\right).
\end{eqnarray*}

However, this further requires approximating integrals of the type
\[
\int q\left(x_{i},u\right)r_{1}\left(u,x_{j}\right)du=\int\int\int\int k\left(z,z'\right)k\left(z,x_{i}\right)k\left(z',u\right)k\left(u,z''\right)k\left(z'',x_{j}\right)\nu_{1}\left(dz\right)\nu_{2}\left(dz'\right)\nu_{1}\left(dz''\right)du,
\]

etc. Thus, while possible in principle, this approach to compute the
variance is cumbersome.

\subsubsection*{A finite dimensional approximation}

To approximate the variance, hence, it is simpler to consider finite-dimensional
approximations to $f_{1}$ and $f_{2}$. Namely, collate $\left\{ x_{i}\right\} $
and $\left\{ y_{j}\right\} $ into a single set of points $\xi$ (note
that we could here take an arbitrary set of points), and consider
finite-dimensional GPs given by
\[
\tilde{f}_{1}=\sum a_{j}k\left(\cdot,\xi_{j}\right),\quad\tilde{f}_{2}=\sum b_{j}k\left(\cdot,\xi_{j}\right),
\]
where we selects distribution of $a$ and $b$ such that evaluations
of $\tilde{f}_{1}$ and $\tilde{f}_{2}$ on $\xi$, $K_{\xi\xi}a$
and $K_{\xi\xi}b$ respectively, have the same distributions as evaluations
of $f_{1}$ and $f_{2}$ on $\xi$. In particular, we take 
\[
a\sim\mathcal{N}\left(K_{\xi\xi}^{-1}m_{1}\left(\xi\right),K_{\xi\xi}^{-1}\mathcal{K}_{1,\xi\xi}K_{\xi\xi}^{-1}\right),\quad b\sim\mathcal{N}\left(K_{\xi\xi}^{-1}m_{2}\left(\xi\right),K_{\xi\xi}^{-1}\mathcal{K}_{2,\xi\xi}K_{\xi\xi}^{-1}\right),
\]
where we denoted by $m_{1}\left(\xi\right)$ a vector such that $\left[m_{1}\left(\xi\right)\right]_{i}=m_{1}\left(\xi_{i}\right)$
and by $\mathcal{K}_{1,\xi\xi}$ a matrix such that $\left[\mathcal{K}_{1,\xi\xi}\right]_{ij}=\kappa_{1}\left(\xi_{i},\xi_{j}\right)$.

Then, clearly 
\begin{eqnarray*}
\left\langle \tilde{f}_{1},\tilde{f}_{2}\right\rangle _{\mathcal{H}_{k}} & = & a^{\top}K_{\xi\xi}b\\
 & = & \left(K_{\xi\xi}^{1/2}a\right)^{\top}\left(K_{\xi\xi}^{1/2}b\right),
\end{eqnarray*}
and now we are left with the problem of computing the mean and the
variance of inner product between two independent Gaussian vectors, as given in Proposition $C.2$. We have
\begin{eqnarray*}
\mathbb{E}\left\langle \tilde{f}_{1},\tilde{f}_{2}\right\rangle _{\mathcal{H}_{k}} & = & \left(K_{\xi\xi}^{1/2}K_{\xi\xi}^{-1}m_{1}\left(\xi\right)\right)^{\top}\left(K_{\xi\xi}^{1/2}K_{\xi\xi}^{-1}m_{2}\left(\xi\right)\right)\\
 & = & m_{1}\left(\xi\right)^{\top}K_{\xi\xi}^{-1}K_{\xi\xi}K_{\xi\xi}^{-1}m_{2}\left(\xi\right)\\
 & = & m_{1}\left(\xi\right)^{\top}K_{\xi\xi}^{-1}m_{2}\left(\xi\right),
\end{eqnarray*}
and 
\begin{eqnarray*}
\text{var}\left\langle \tilde{f_{1}},\tilde{f_{2}}\right\rangle _{\mathcal{H}_{k}} & = & \left(K_{\xi\xi}^{1/2}K_{\xi\xi}^{-1}m_{1}\left(\xi\right)\right)^{\top}K_{\xi\xi}^{-1/2}\mathcal{K}_{2,\xi\xi}K_{\xi\xi}^{-1/2}\left(K_{\xi\xi}^{1/2}K_{\xi\xi}^{-1}m_{1}\left(\xi\right)\right)\\
 & + & \left(K_{\xi\xi}^{1/2}K_{\xi\xi}^{-1}m_{2}\left(\xi\right)\right)^{\top}K_{\xi\xi}^{-1/2}\mathcal{K}_{1,\xi\xi}K_{\xi\xi}^{-1/2}\left(K_{\xi\xi}^{1/2}K_{\xi\xi}^{-1}m_{2}\left(\xi\right)\right)\\
 & + & \text{tr}\left(K_{\xi\xi}^{-1/2}\mathcal{K}_{1,\xi\xi}K_{\xi\xi}^{-1/2}K_{\xi\xi}^{-1/2}\mathcal{K}_{2,\xi\xi}K_{\xi\xi}^{-1/2}\right)\\
 & = & m_{1}\left(\xi\right)^{\top}K_{\xi\xi}^{-1}\mathcal{K}_{2,\xi\xi}K_{\xi\xi}^{-1}m_{1}\left(\xi\right)\\
 & + & m_{2}\left(\xi\right)^{\top}K_{\xi\xi}^{-1}\mathcal{K}_{1,\xi\xi}K_{\xi\xi}^{-1}m_{2}\left(\xi\right)\\
 & + & \text{tr}\left(\mathcal{K}_{1,\xi\xi}K_{\xi\xi}^{-1}\mathcal{K}_{2,\xi\xi}K_{\xi\xi}^{-1}\right).
\end{eqnarray*}

\subsubsection*{Coming back to BayesIMP}
Now coming back to the derivation of BayesIMP. We will first provide two finite approximation of $f$ and $\mu_{gp}^{do}(x, \cdot)$ in the following two propositions. Recall these finite approximations are set up such that they match the distributions of evaluations of $f$ and $\mu_{gp}^{do}$ at $\hat{\bfy} = [\bfy^\top \hspace{0.1cm} \tilbfy^\top]^\top$. The latter thus act as landmark points for the finite dimensional approximations.

\begin{customprop}{C.3}[Finite dimensional approximation of $f$]
Let $\hat{\bfy} = [\bfy^\top \hspace{0.1cm} \tilbfy^\top]^\top$ be the concatenation of $\bfy$ and $
\tilbfy$. We can approximate $f$ with ,
\begin{equation}
    \tilde{f}|\bft \sim N(m_{\tilde{f}}, \Sigma_{\tilde{f}})
\end{equation}
where,
\begin{align}
    m_{\tilde{f}} &= \Phi_{\hat{\bfy}}K_{\hat{\bfy}\hat{\bfy}}^{-1}R_{\hat{\bfy}\tilbfy}(R_{\tilbfy\tilbfy} + \lambda_f I)^{-1}\bft \\
    \Sigma_{\tilde{f}} &= \Phi_{\hat{\bfy}}K_{\hat{\bfy}\hat{\bfy}}^{-1} \bar{R}_{\hat{\bfy}\hat{\bfy}} K_{\hat{\bfy}\hat{\bfy}}^{-1} \Phi_{\hat{\bfy}}^\top
\end{align}
and $\bar{R}_{\hat{\bfy}\hat{\bfy}} = R_{\hat{\bfy}\hat{\bfy}} - R_{\hat{\bfy}\tilbfy}(R_{\tilbfy\tilbfy} + \lambda_f I)^{-1} R_{\tilbfy\hat{\bfy}}$.
\end{customprop}
Similarly for $\mu_{gp}^{do}(x, \cdot)$, we have the following
\begin{customprop}{C.4}[Finite dimensional approximation of $\mu_{gp}^{do}(x, \cdot)$]\label{proposition: induced dist on mu}
Let $\hat{\bfy} = [\bfy^\top \hspace{0.1cm} \tilbfy^\top]^\top$ be the concatenation of $\bfy$ and $
\tilbfy$. We can approximate $\mu_{gp}^{do}(x, \cdot)$ with ,
\begin{equation}
    \tilde{\mu}_{gp}^{do}(x, \cdot)|\operatorname{vec}(K_{\bfy\bfy}) \sim N(m_{\tilde{\mu}}, \Sigma_{\tilde{\mu}})
\end{equation}
where,
\begin{align}
    m_{\tilde{\mu}} &= \Phi_{\hat{\bfy}}K_{\hat{\bfy}\hat{\bfy}}^{-1}R_{\hat{\bfy}\bfy}R_{\bfy\bfy}^{-1}K_{\bfy \bfy}(K_{\Omega_x} + \lambda I)^{-1}\Phi_{\Omega_x}(x) \\
    \Sigma_{\tilde{\mu}} &=     \Phi_{\hat{\bfy}}K_{\hat{\bfy}\hat{\bfy}}^{-1}K^{\mu}_{\hat{\bfy}\hat{\bfy}} K_{\hat{\bfy}\hat{\bfy}}^{-1}\Phi_{\hat{\bfy}}^\top
\end{align}
where $K^{\mu}_{\hat{\bfy}\hat{\bfy}} = \Phi_{\Omega_x}(x)^\top \Phi_{\Omega_x}(x)R_{\hat{\bfy}\hat{\bfy}} - \big(\Phi_{\Omega_x}(x)^\top(K_{\Omega_x}+\lambda I)^{-1}\Phi_{\Omega_x}(x)\big) R_{\hat{\bfy}\bfy} R_{\bfy\bfy}^{-1}R_{\bfy\hat{\bfy}}$
\end{customprop}
Now we have everything we need to derive the main algorithm in our paper, the \textsc{BayesIMP}. Note that we did not introduce the $\mu_{gp}^{do}$ notation in the main text to avoid confusion as we did not have space to properly define $\mu_{gp}^{do}$.

\begin{customprop}{\ref{proposition: prop: BayesIMP} }[\textbf{\textsc{BayesIMP}}]
Let $f$ and $\mu_{Y|do(X)}$ be \textsc{GP}s learnt as above. Denote $\tilde{f}$ and $\tilde{\mu}_{Y|do(X)}$ as the finite dimensional approximation of $f$ and $\mu_{Y|do(X)}$ respectively. Then $\tilde{g} = \langle \tilde f, \tilde \mu_{Y|do(X)} \rangle$ has the following mean and covariance:
\begin{align}
\small
    m_3(x) &= E_x K_{\bfy\hat{\bfy}}K_{\hat{\bfy}\hat{\bfy}}^{-1}R_{\hat{\bfy}\tilbfy}(R_{\tilbfy\tilbfy} + \lambda_f I)^{-1} \bft \\
    %\kapp_3(x, x') &= \Upsilon_1 + \Upsilon_2 + \Upsilon_3
    \kappa_3(x,x') &= \underbrace{E_x\Theta_1^\top \tilde{R}_{\hat{\bfy}\hat{\bfy}} \Theta_1E_{x'}^\top}_{\text{Uncertainty from } \cD_1 } + \underbrace{\Theta_2^{(a)}F_{xx'} - \Theta_2^{(b)}G_{xx'}}_{\text{Uncertainty from } \cD_2} + \underbrace{\Theta_3^{(a)}F_{xx'} - \Theta_3^{(b)}G_{xx'}}_{\text{Uncertainty from Interaction}}
\end{align}
where $E_x = \Phi_{\Omega_x}(x)^\top (K_{\Omega_x} + \lambda I)^{-1}, F_{xx'}=\Phi_{\Omega_x}(x)^\top \Phi_{\Omega_x}(x'), G_{xx'} = \Phi_{\Omega_x}(x)^\top(K_{\Omega_x} + \lambda I)^{-1}\Phi_{\Omega_x}(x')$, and $\Theta_1 = K_{\hat{\bfy}\hat{\bfy}}^{-1}R_{\hat{\bfy}\bfy}R_{\bfy\bfy}^{-1}K_{\bfy\bfy}$, $\Theta_2^{(a)} = \Theta_4^\top R_{\hat{\bfy}\hat{\bfy}}\Theta_4, \Theta_2^{(b)} =  \Theta_4^\top R_{\hat{\bfy}\bfy}R_{\bfy\bfy}^{-1}R_{\bfy\hat{\bfy}}\Theta_4$ and $\Theta_3^{(a)} = tr(K_{\hat{\bfy} \hat{\bfy}}^{-1}R_{\hat{\bfy}\hat{\bfy}}K_{\hat{\bfy}\hat{\bfy}}^{-1}\bar{R}_{\hat{\bfy}\hat{\bfy}}), \Theta_3^{(b)} = tr(R_{\hat{\bfy}\bfy}R_{\bfy\bfy}^{-1}R_{\bfy\hat{\bfy}}K_{\hat{\bfy}\hat{\bfy}}^{-1}\bar{R}_{\hat{\bfy}\hat{\bfy}}K_{\hat{\bfy}\hat{\bfy}}^{-1})$ and $\Theta_4 = K_{\hat{\bfy}\hat{\bfy}}^{-1}R_{\hat{\bfy}\tilbfy}(K_{\tilbfy\tilbfy}+\lambda_f)^{-1}\bft$. $\bar{R}_{\hat{\bfy}\hat{\bfy}}$ is the posterior covariance of $f$ evaluated at $\hat{\bfy}$ 
\end{customprop}

\begin{proof}[Proof of Proposition \ref{proposition: prop: BayesIMP}]
Since $\tilde{g}=\langle \tilde{f}, \tilde{\mu}_{gp}^{do}\rangle$ is an inner product between two finite dimensional GPs, we know the variance (as given by Proposition $C.2$) is characterised by,
\begin{equation}
    var(g) =  m_{\tilde{\mu}}^\top \Sigma_{\tilde{f}} m_{\tilde{\mu}} + m_{\tilde{f}}^\top \Sigma_{\tilde{\mu}} m_{\tilde{f}} + \operatorname{tr}(\Sigma_{\tilde{f}}\Sigma_{\tilde{\mu}})
\end{equation}
Expanding out each terms we get Proposition \ref{proposition: prop: BayesIMP}:
\begin{align}
\small
    m_{\tilmu}^\top \Sigma_{\tilf} m_{\tilmu} &= E_x\Theta_1^\top \tilde{R}_{\hat{\bfy}\hat{\bfy}} \Theta_1E_{x'}^\top \\
    m_{\tilde{f}}^\top \Sigma_{\tilde{\mu}} m_{\tilde{f}} &= \Theta_2^{(a)}F_{xx'} - \Theta_2^{(b)}G_{xx'} \\
\end{align}
while the first two terms resembles the uncertainty obtained from \textsc{IMP} and \textsc{BayesIME}, the trace term is new and we will expand it out here,
\begin{align}
    \operatorname{tr}(\Sigma_{\tilf}\Sigma_{\tilmu}) &= \operatorname{tr}\Big(\Phi_{\hat{\bfy}}K_{\hat{\bfy}\hat{\bfy}}^{-1}K^{\mu}_{\hat{\bfy}\hat{\bfy}} K_{\hat{\bfy}\hat{\bfy}}^{-1}\Phi_{\hat{\bfy}}^\top\Phi_{\hat{\bfy}}K_{\hat{\bfy}\hat{\bfy}}^{-1} \bar{R}_{\hat{\bfy}\hat{\bfy}} K_{\hat{\bfy}\hat{\bfy}}^{-1} \Phi_{\hat{\bfy}}^\top\Big) \\
    &= \operatorname{tr}\Big(K_{\hat{\bfy}\hat{\bfy}}^{-1}K_{\hat{\bfy}\hat{\bfy}}^{\mu}K_{\hat{\bfy}\hat{\bfy}}^{-1}\bar{R}_{\hat{\bfy}\hat{\bfy}} \Big) \\
    &= \operatorname{tr}\Big(K_{\hat{\bfy}\hat{\bfy}}^{-1}\big(F_{xx'}R_{\hat{\bfy}\hat{\bfy}} - G_{xx'}R_{\hat{\bfy}\bfy}R_{\bfy\bfy}^{-1}R_{\bfy\hat{\bfy}}\big)K_{\hat{\bfy}\hat{\bfy}}^{-1}\bar{R}_{\hat{\bfy}\hat{\bfy}}\Big) \\
    &= \Theta_3^{(a)} F_{xx'} - \Theta_3^{(b)} G_{xx'}
\end{align}

\end{proof}
\newpage

\section{Details on Experimental setup}
\subsection{Details on Ablation Study}
\subsubsection{Data Generating Process}
We use the following causal graphs, $X \xrightarrow[]{} Y$ and $Y \xrightarrow[]{} T$, to demonstrate a simple scenario for our data fusion setting. As linking functions, we used for $\cD_1$, $Y = x cos(\pi x) + \epsilon_1$ and for $\cD_2$, $T = 0.5* y*cos(y) + \epsilon_2$. where $\epsilon_i \sim \mathcal{N}(0, \sigma_i)$. Here below we plotted the data for illustration purposes.
\begin{figure}[!htp]
    \centering
    \includegraphics[width=\textwidth]{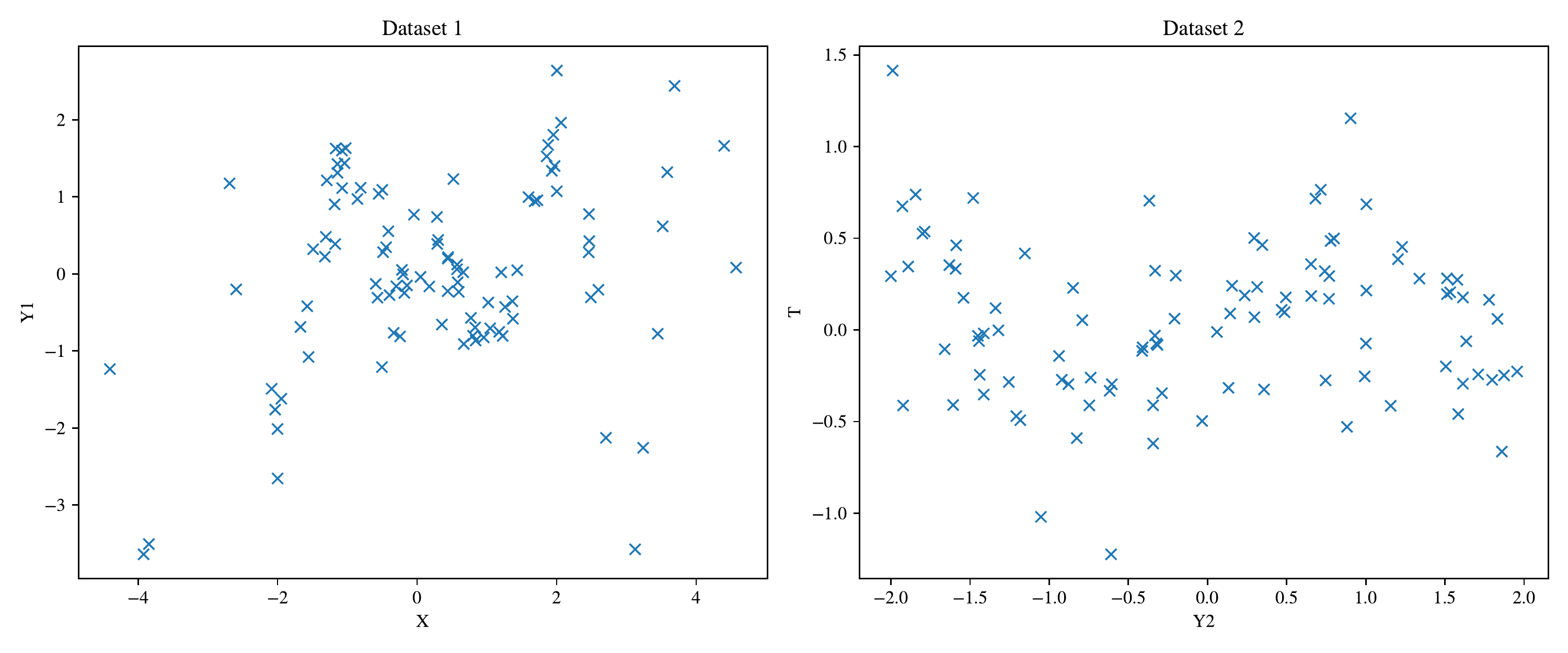}
    \caption{(Left) Illustration of $\cD_1$ (Right) Illustration of $\cD_2$}
\end{figure}
\subsubsection{Explanation on the extrapolation effect}
In the main text we referred to the case where \textsc{IMP} is better than \textsc{BayesIME} as \textbf{extrapolation effect}. We note from the figure above that in $\cD_1$ we have $x$ around $-4$ being mapped onto $y$ values around $-3$. Note however, that in $\cD_2$, we do not observe any values $\tilde{Y}$ below $-2$. Hence, because \textsc{IMP} uses a GP model for $\cD_2$ we are able to account for this mismatch in support and hence attribute more uncertainty to this region, i.e. we see the spike in uncertainty in Fig.\ref{fig: prelim} for \textsc{IMP}.
\subsubsection{Calibration Plots}
To investigate the accuracy of the uncertainty quantification in the proposed methods, we perform a (frequentist) calibration analysis of the credible intervals stemming from each method. Fig. \ref{fig:calibration} gives the calibration plots of the Sampling methods (sampling-based method of \cite{aglietti2020causal}) as well as the three proposed methods. On the x-axis is the portion of the posterior mass, corresponding to the width of the credible interval. We will interpret that as a nominal coverage probability of the true function values. On the y-axis is the true coverage probability estimated using the percentage of the times true function values do lie within the corresponding credible intervals. A perfectly calibrated method should have nominal coverage probability equal to the true coverage probability, i.e. being closer to the diagonal line is better.
\begin{figure}[!htp]
    \centering
    \includegraphics[width=0.7\textwidth]{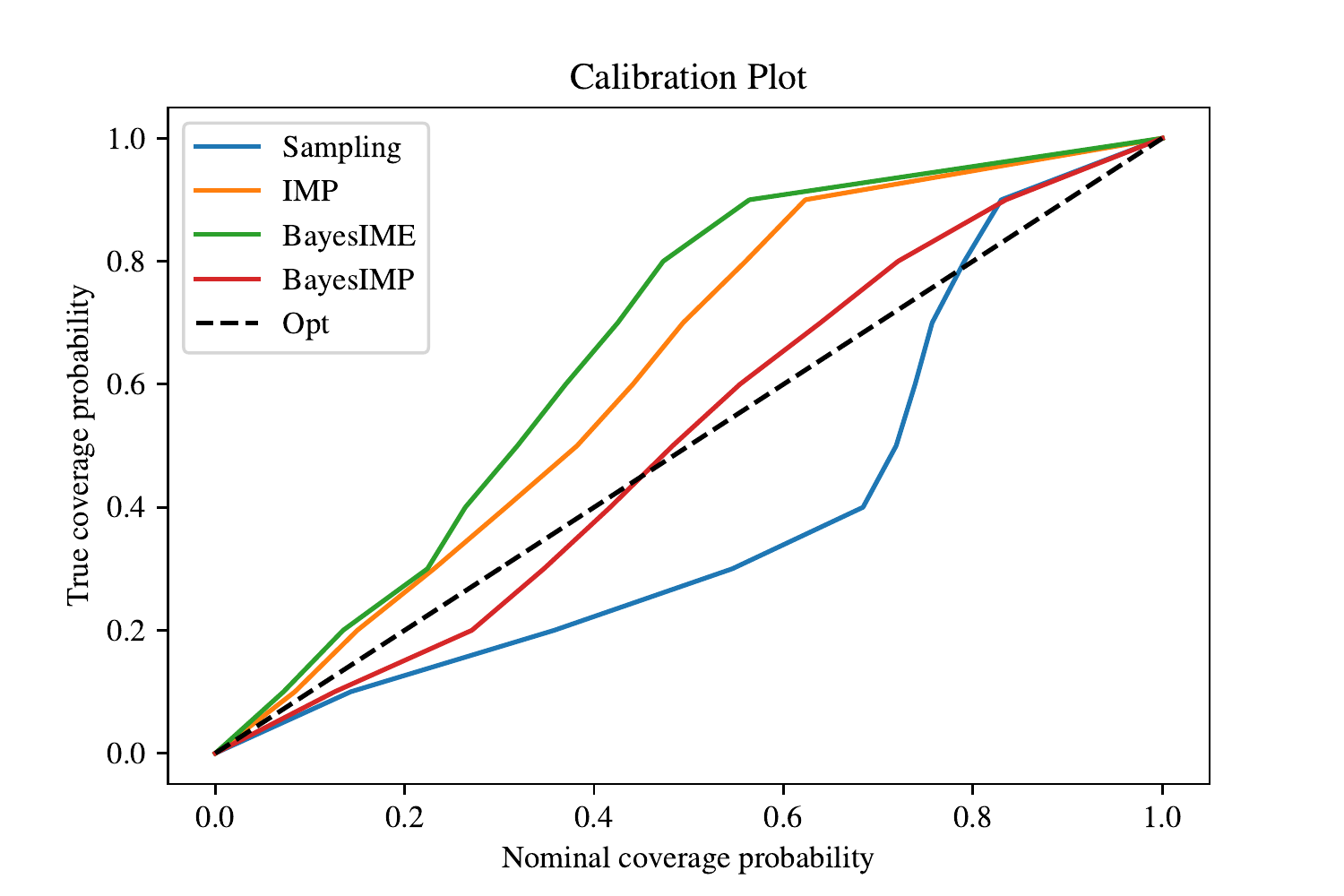}
    \caption{Calibration plots of Sampling method as well as our 3 proposed methods. We clearly see that \textsc{BayesIMP} is the best calibrated method amongst all other methods.}
    \label{fig:calibration}
\end{figure}
\subsection{Details on Synthetic Data experiments}
\subsubsection{Data Generating Process for simple synthetic dataset}
For the first simple synthetic dataset (See Fig.\ref{fig: synth1} (Top)) we used the following data generating graph is defined as.
\begin{itemize}
    \item $X \xrightarrow[]{} U: U = 2 * X + \epsilon$
    \item $Z \xrightarrow[]{} X: X = 3 * \cos(Z)+\epsilon$
    \item $\{Z, U\} \xrightarrow[]{} Y: Y = U + \exp(-Z)+\epsilon$
    \item $Y \xrightarrow[]{} T: T = \cos(Y) - \exp(-y/20)+\epsilon$
\end{itemize}
where $\epsilon \sim  \mathcal{N}(0, \sigma^2)$ and $Z \sim \mathcal{U}[-4, 4]$, where for $\cD_2$ we have that $\tilde{Y} \sim \mathcal{U}[-10, 10]$.
In addition, with probability $\pi=1/2$ we shift $U$ by $+1$ horizontally and $-3$ vertically to thus create the multimodality in the data. In order to generate from the interventional distribution, we simply remove the edge from $Z \xrightarrow[]{} X$ and fix the value of $x$.

\subsubsection{Data Generating Process for harder synthetic dataset from \cite{aglietti2020causal}}
For the first simple synthetic dataset (Fig.\ref{fig: synth1}(Bottom)) we used the same data generating format as in \cite{aglietti2020causal}.
\begin{itemize}
    \item $U_1 =\epsilon_1$
    \item $U_2 =\epsilon_2$
    \item $F =\epsilon_3$
    \item $A =F^2 + U_1 + \epsilon_A$
    \item $B = U_2 + \epsilon_B$
    \item $C = \exp(-B) + \epsilon_C$
    \item $D = \exp(-C)/10 + \epsilon_D$
    \item $E = \cos(A) + C/10 \epsilon_E$
    \item $Y_1 = \cos(D) + \sin(E) + U_1 + U_2$
    \item $Y_2 = \cos(D) + \sin(E) + U_1 + U_2 +  2\pi$
    \item $T = 6 * \sin(3 * Y) + \epsilon$
\end{itemize}
where the noise is fixed to be $\mathcal{N}(0, 1)$ and where we switch with $\pi=1/2$ from mode $Y_1$ and $Y_2$, where $\tilde{Y} \sim \mathcal{U}[-2, 9]$ for $\cD_2$.
\subsection{Details on Healthcare Data experiments}
\subsubsection{Data Generating Process}
For the healthcare dataset, $\cD_1$, (Fig.\ref{fig:medical example}) we used the same data generating format as in \cite{aglietti2020causal} with the difference that we make \emph{statin} continuous and increased the age range.
\begin{itemize}
    \item $age =\mathcal{U}[15, 75]$
    \item $bmi =\mathcal{N}(27- 0.01*age, 0.7)$
    \item $aspirin =\mathcal{\sigma}(-8.0 + 0.1*age + 0.03 *bmi)$
    \item $statin = -13 + 0.1*age + 0.2*bmi$
    \item $cancer =\mathcal{\sigma}(2.2 - 0.05 * age + 0.01 * bmi - 0.04*statin + 0.02 * aspirin)$
    \item $PSA =\mathcal{N}(6.8 + 0.04*age - 0.15*bmi-0.6*statin + 0.55*aspirin + cancer, 0.4)$
\end{itemize}
As for the second dataset, $\cD_2$ we firstly fit a GP on the data collected from \cite{stamey1989prostate}. Once we have the posterior GP, we can then use it as a generator for the $\cD_2$ as it takes as input $PSA$. This generator hence acts as a link between $\cD_1$ and $\cD_2$. This way we are able to create a simulator that allows us to obtain samples from $\EE[\emph{Cancer volume}| do(Statin)]$ for our causal BO setup.

\subsection{Bayesian Optimisation experiments with \textsc{IMP} and \textsc{BayesIME}}
\begin{figure}[!htp]
    \centering
    \includegraphics[width=0.3\textwidth]{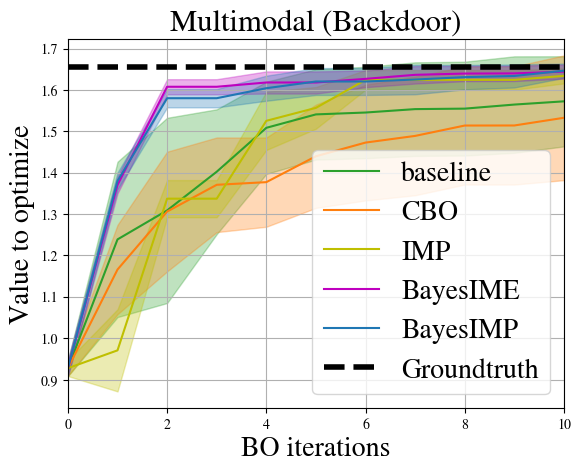}
    \includegraphics[width=0.3\textwidth]{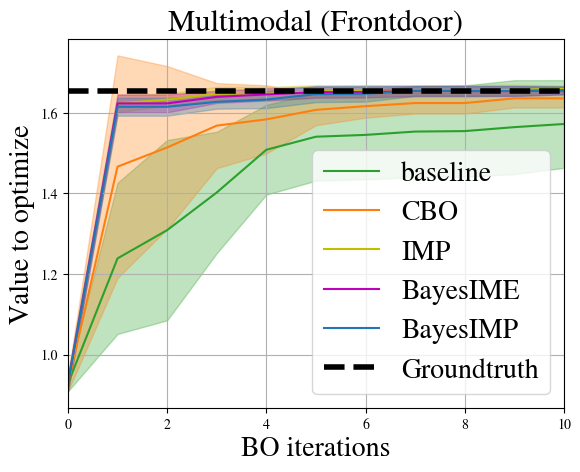}
    \includegraphics[width=0.3\textwidth]{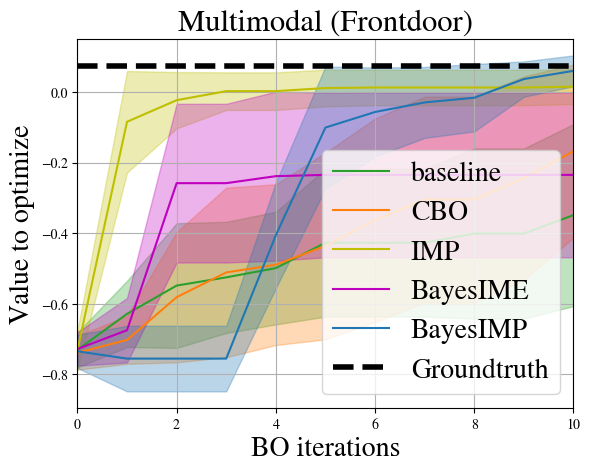}
    \caption{(Left) Simple graph using backdoor adjustment (Middle) Simple graph using front-door adjustment (Right) Harder graph using front-door adjustment. \textsc{BayesIMP} strikes the right balance between \textsc{IMP} and \textsc{BayesIME} and all three perform better than \textsc{CBO} and the GP baseline.}\label{fig: app_figs_methods}
\end{figure}

The main text compares \textsc{BayesIMP} to \textsc{CBO} and the baseline GP with no learnt prior in the Bayesian Optimisation experiments. Here, we include \textsc{IMP} and \textsc{BayesIME} (i.e. simplified versions of \textsc{BayesIMP} that account for only one source of uncertainty each) in those comparisons. We see from Fig.\ref{fig: app_figs_methods} that \textsc{BayesIMP} is comparable to \textsc{IMP} and \textsc{BayesIME} in most cases. While \textsc{BayesIMP} is not the best performing method in every scenario, it does hit a good middle ground between the first two proposed methods. For Fig.\ref{fig: app_figs_methods} (Left, Middle) we used $N=100$ and $M=50$. In the left figure, \textsc{BayesIME} and \textsc{BayesIMP} are very similar, whereas \textsc{IMP} is considerably worst. In the middle figure, all methods seems to perform well without much difference. In the right figure, we have $N=500$ and $M=50$ and this is a case where \textsc{IMP} is best, while \textsc{BayesIME} appears to get stuck in a local optimum (recall that \textsc{BayesIME} does not take into account uncertainty in $\cD_2$ where there is little data). We note that all three methods converge faster than the current SOTA \textsc{CBO}.

\end{document}